\documentclass[a4paper,11pt,oneside]{article}
\usepackage[margin=1in]{geometry}
\usepackage{amsmath, amsthm, amssymb}
\usepackage{hyperref}
\usepackage{mathtools}
\usepackage{natbib}
\usepackage[utf8]{inputenc}

\hypersetup{
colorlinks=true,
linkcolor=red,          %
citecolor=blue,         %
filecolor=magenta,      %
urlcolor=cyan           %
}

\usepackage[utf8]{inputenc} %
\usepackage[T1]{fontenc}    %
\usepackage{hyperref}       %
\usepackage{url}            %
\usepackage{booktabs}       %
\usepackage{amsfonts}       %
\usepackage{nicefrac}       %
\usepackage{microtype}      %

\usepackage{mathtools}
\usepackage{amsthm, amssymb}

\newtheorem{theorem}{Theorem}
\newtheorem{lemma}[theorem]{Lemma}

\usepackage{amsthm}
\usepackage{mathtools}
\usepackage{bbm}

\usepackage{commands_arxiv}

\title{Metric-valued regression}
\author{Dan {Tsir Cohen} and Aryeh Kontorovich \\\small{\texttt{dantsir@post.bgu.ac.il,karyeh@cs.bgu.ac.il}}}

\begin{document}

\maketitle

\begin{abstract}

We propose an efficient algorithm for learning
mappings between two metric spaces, $\X$ and $\Y$. Our procedure
is strongly Bayes-consistent whenever $\X$ and $\Y$
are topologically
separable
and $\Y$ is ``bounded in expectation'' (our term; the separability
assumption
can be somewhat
weakened).
At this level of generality,
ours is the first such learnability result
for unbounded loss
in the agnostic setting.
Our technique 
is based on metric medoids (a variant of
Fr\'echet means)
and
presents a significant departure from existing
methods, which, as we demonstrate, fail to achieve Bayes-consistency
on general instance- and label-space metrics. Our proofs
introduce the technique of {\em semi-stable compression},
which may be of independent interest.

\end{abstract}

\section{Introduction}
\label{SEC:intro}
Regression and multiclass classification fall under the rubric of supervised prediction from %
labeled examples.
The chief difference between the two
is that classification typically
assumes the discrete metric on the label space (captured by the 0-1 loss),
while regression (at least with absolute loss\footnote{
Quadratic loss can be captured by 
the
{\em inframetrics}
\citep{Fraigniaud08}
or
{\em near-metrics}
\citep{steve-optimistic}.})
implicitly assumes
the standard metric over the real-valued labels.
In this paper, we 
study the considerably more general setting
of {\em metric-valued regression}, where the labels reside
in an arbitrary metric space.
This setting subsumes both multiclass classification and real-valued regression, and strictly generalizes these.

We consider the following fundamental learning problem:
the instance space $\X$ is endowed with a metric $\rho$
and the label space $\Y$ with a metric $\ell$.
The learner receives a training sample $(X_i,Y_i)$, $i\in[n]$,
drawn iid from an unknown distribution $\bmu$ on $\X\times\Y$.
The learner's goal is to (efficiently) produce a hypothesis
$f_n:\X\to\Y$, based on the labeled sample, so as to minimize
the {\em risk}
$\risk(f_n):=\E_{(X,Y)\sim\bmu}\ell(f_n(X),Y)$. In particular, we say that
the learning procedure is strongly universally Bayes-consistent
if,
for every $\bmu$,
we have that
$\risk(f_n)\to \risk(f^*)$ almost surely as $n\to\infty$,
 where
$f^*$ is the minimizer of $\risk(\cdot)$ over all measurable $f:\X\to\Y$.

\paragraph{Our contribution.}
We propose a novel algorithm, \medoidnet, for learning
in the metric-valued regression setting. 
To our knowledge, this
is the first 
strong Bayes-consistency result
for unbounded loss with agnostic noise.
While inspired by
the \hkswname\ algorithm of \citet{hksw21},
the extension from 
0-1 loss
to arbitrary metrics
required non-trivial modifications to the learning
procedure and the risk analysis; 
a
detailed account
of the 
similarities and innovations
is provided in
Section~\ref{sec:main-res}.
We show that under 
quite general,
natural
conditions on the metric spaces
$(\X,\rho)$ and $(\Y,\ell)$,
our algorithm is strongly universally Bayes-consistent.
The structural assumptions on $\X$ 
and $\Y$
are truly minimalistic:
we require them to be separable metric spaces,
and for $\Y$ to be {\em bounded in expectation}:
$\E_{(X,Y)\sim\bmu}\ell(y_0,Y)<\infty$
for some $y_0\in Y$.
A byproduct of our analysis is the introduction
of the {\em semi-stable compression} technique, which
may be of independent interest.

\paragraph{Related work.}
The regression setting with labels
residing in a metric space other than $\R$
is relatively uncommon;
such works include
\citet{ferraty11}
and
\citet{DBLP:journals/corr/abs-1905-11930},
who study the Banach- and Hilbert-space
valued cases, respectively,
as well as the more recent results discussed below.
Our work builds on
\citet{hksw21},
who gave a complete characterization
of the 
metric spaces 
$(\X,\rho)$
for which there
exists a
strong Universal Bayes-Consistent (UBC)
learner, 
where $\Y=\N$
is endowed with the discrete metric.
They also provided an algorithm, \hkswname, which achieves strong
UBC whenever the latter is achievable by {\em any} learner (also
provided therein is a comprehensive literature review of the 
strengths
and limitations of previous metric-space methods, including $k$-NN).
\citet{gyorfi-weiss21}
followed up with a simplified algorithm, 
\protonn, which, in addition to enjoying all of the properties of
\hkswname, is also strongly UBC in $L_1$ for unbounded real labels
$\Y=\R$, as long as 
$\E|Y|<\infty$;
our boundedness in expectation 
generalizes
this condition.

\citet{steve-optimistic} introduced the
very general
paradigm of ``learning whenever learning is possible,''
which extends the iid setting
to essentially the broadest possible class
of random processes. That work dealt mostly with the realizable (noiseless) case and bounded losses, though certain kinds of noise
were considered in Section 9 therein.
A series of recent preprints 
followed:
\citet{BlanCos21,Blanchard21,Hanneke21,BCH21}.
These 
also
study
sampling processes
far more general than iid, %
but
consider
label-loss structures that are bounded, or noiseless, or both.
\citet{BlanCos21}, for example,
provide a reduction from the metric-valued regression problem to the binary classification problem for the realizable setting with bounded loss. Another aspect in which the above works are incomparable to ours is 
their use of
{\em non-algorithmic} learning procedures: more in the spirit of an existence proof, these
involve
non-constructive operations
such as
enumerating elements of a $\sigma$-algebra.\footnote{
So as not to get bogged down
with computability issues
over continuous inputs,
we only claim
full algorithmic constructivity
for countable $\Y$.
When $\Y$ is merely separable,
we 
assume access to an oracle
for computing
$\eps$-nets over $\Y$.
Such an oracle is easily constructed for, e.g., $\Y=\R^d$.
}

In the special case of the singleton
$\X=\set{x}$ 
and a general
metric space $(\Y,\ell)$,
the consistency of various
Fr\'echet means (which are naturally related to medoids) has been recently
examined by \citet{evans2020strong,schotz2021strong}.
More tangentially related
works include
\citet{DBLP:conf/icml/MorvantKR12},
who,
in a PAC-Bayesian setting,
gave multiclass risk bounds with a {\em confusion matrix}
error structure, which is close in spirit to assuming a metric
on the label set. The assumptions there are fairly restrictive
(every label must appear at least a constant number of times
in the sample), and no learning procedure or Bayes-consistency result
was provided.
On the algorithmic front, our procedure
partitions the instance space $\X$ into Voronoi
cells and chooses the label $y\in\Y$ for each cell based
on a variant of the {\em medoid} principle. A number of
loosely medoid-based learning algorithms have been proposed
\citep{derlaan03,GottliebKK13tcs+alt,DBLP:conf/aistats/NewlingF17,DBLP:conf/nips/BaharavT19}, but our approach is distinct from all of these,
in that we compute medoids in the
{\em label} (rather than instance) space.

Finally, {\em stable compression} was a technique introduced by
\citet{BousquetHMZ20}
for the realizable case and extended to the agnostic case
by \citet{DBLP:conf/alt/HannekeK21}.
We introduce a {\em semi-stable} variant for both cases by allowing additional
{side information};
only the compression set (and not the side information) is required to satisfy the
stability condition of 
\citeauthor{BousquetHMZ20}

\section{Main results and overview of techniques}
\label{sec:main-res}

Our main result is the existence of a strong Universal Bayes-Consistent (UBC) learner for metric-valued regression. 

\begin{theorem}
\label{thm:main}
There exists a learning algorithm, \medoidnet, with the following property.
Let $(\X,\rho)$ and $(\Y,\ell)$ be separable metric spaces
endowed with
a distribution $\bmu$
supported on the product Borel $\sigma$-algebra of $\X\times\Y$,
such that
$\Y$ is {\em bounded in expectation}:
$\E_{(X,Y)\sim\bmu}\ell(y_0,Y)<\infty$
for some $y_0\in \Y$. 
Given a training sample $(X_i,Y_i)_{i\in[n]}\sim\bmu^n$ as input,
\medoidnet\ outputs a predictor $f_n:\X\to\Y$
that is strongly universally Bayes-consistent:
$\risk(f_n)\to R^*$
almost surely (under $\bmu$)
as $n\to\infty$,
where 
$\risk(f)=\E_{(X,Y)\sim\bmu}\ell(f(X),Y)$
is the risk and $R^*$ the 
Bayes-optimal risk
(minimum risk achieved by any measurable $f$).
\end{theorem}

The proof proceeds via a sequence of incremental results, culminating in \thmref{comp-consist4},
which is a restatement of \thmref{main}.
A few remarks regarding our assumptions on 
$(\X,\rho)$
and
$(\Y,\ell)$
are in order. 
As per \citet{hksw21}, 
the assumption of separability
may be weakened to 
{\em essential}
separability (ES):
this is the condition that the 
support of $\bmu$
is contained in a separable subspace.
It was shown therein that the ES condition (on $\X$)
is also necessary for {\em any} learner to succeed,
and observed that
any actual metric space encountered in practice will be ES;
in fact, the existence of non-ES metric spaces is widely believed
to be independent of ZFC.
For countable $\Y$,
a variant of \thmref{main}
--- namely, \thmref{comp-consist2} ---
holds for {\em any} 
bounded 
loss function $\ell:\Y\times\Y\to[0,L]$;
no metric structure is necessary.
Finally, a word about the computational efficiency
of \medoidnet. The latter, conceptually, consists of two stages:
(I) computing a $\gamma$-net on the finite 
training
sample (residing in $\X$)
and (II) for each Voronoi cell $C$ induced on the sample by the $\gamma$-net,
finding a {\em medoid} $y\in\Y$ that minimizes $\sum_{c\in C}\ell(c,y)$.
The $\gamma$-net in stage (I)
can 
indeed
be 
efficiently
constructed
\citep{DBLP:journals/tit/GottliebKK14+colt,DBLP:journals/jmlr/KpotufeV17}.
At stage (II), 
\medoidnet\
truncates $\Y$ 
adaptively
to some finite $\Y'$ 
over which
the medoid can always be computed in a runtime linear in $|\Y'|$.
When
additional structural information regarding $\Y$
is available, it may be leveraged to obtain more efficient
medoid oracles.

To illustrate our significant departure from previous
techniques, let us provide some simple examples
where those fail to be Bayes-consistent for
various simple label metrics.
Let $\X=\set{0}$ be the trivial singleton metric space
and $\Y=\set{a,b,c,o}$ be the label space endowed with
the metric $\ell(a,b)=\ell(b,c)=
\ell(c,a)=
1$; $\ell(o,a)=
\ell(o,b)=
\ell(o,c)=1/2
$,
and let the distribution $\bmu$ be such that
$
\P_{(X,Y)\sim\bmu}(Y=a)
=
\P_{(X,Y)\sim\bmu}(Y=b)
=
\P_{(X,Y)\sim\bmu}(Y=c)
=1/3
$.
We observe that any {\em majority-vote}
based method, such as $k$-NN,
which takes a vote among the $k$ nearest neighbors
\citep{gyorfi:02},
or 
\hkswname,
which takes a vote within each Voronoi cell
\citep{hksw21},
or the memory-based techniques of
\citet{BlanCos21,Blanchard21,
BCH21},
or the hybrid approach of \citet{gyorfi-weiss21}
cannot achieve Bayes-consistency in this case
---
for the simple reason that they can only output
the {\em observed} labels $a,b,c$
(and hence, at best, achieve an asymptotic risk of $2/3$), while
the Bayes-optimal predictor 
$f^*\equiv o$
achieves $\risk(f^*)=1/2$.

The necessity of predicting labels that
never occurred in the sample required overcoming
a subtle challenge not present in 
\citet{hksw21}. As in that work, we obtain finite-sample generalization bounds via a sample compression scheme. 
The latter selects a sub-sample
$S_I=(X_i,Y_i)_{i\in I\subset[n]}$, based on which
the predictor will be constructed.
To mitigate the noise, we occasionally wish
to relabel a point $X_i\in S_I$ with a label other than $Y_i$. One could do this using
$b$ bits of side information, but 
\citeauthor{hksw21} sidestepped this issue by doubling the compression set size: the first $k$ pairs $(X_i,Y_i)$ indicate which $X$'s to use (their $Y$'s are discarded) and the second $k$
pairs indicate how to label those first $k$ 
points $X_i$. This stratagem is not applicable
when we wish to relabel an $X$ with a $Y\in\Y$
not occurring in the sample.
Instead, \medoidnet\ adaptively truncates $\Y$
to a finite $\Y_n$, whose elements can be described in $b(n)=\log_2|\Y_n|$ bits of side
information. 
This solves two problems simultaneously: the concentration inequalities
we invoke require a bounded range, and our compression schemes require bounded side information.
We introduce a semi-stable
variant of the
{\em stable}
compression scheme
\citep{BousquetHMZ20,DBLP:conf/alt/HannekeK21}
to analyze the behavior of our truncated medoid.

Finally, we recall
the family of regression techniques based on
Lipchitz-extension 
and observe that it is not
suitable for learning general
metric-to-metric mappings.
A binary classifier
based on the McShane-Whitney
extension theorem
was shown to be Bayes-consistent
\citep{kontorovich2014bayes};
this technique was also applied
by
\citet{GottliebKK13-IEEE,DBLP:conf/colt/AshlagiGK21}
to real-valued regression.
When $\X$ and $\Y$
are both Hilbert spaces,
the Kirszbraun extension
theorem
likewise provides a basis
for a regression algorithm
\citep{DBLP:journals/corr/abs-1905-11930}.
(While 
the latter 
three
works do not prove
Bayes-consistency, 
the finite-sample generalization
bounds provided therein
are likely straightforwardly
adaptable to such a result
via an appropriate regularization schedule.)
Unfortunately, Lipschitz extension
is limited to
a small number of
metric spaces with a special structure;
besides the aforementioned cases, 
\citet{naor-tree12}
established one
for $\X$
a
locally compact length space 
and $\Y$
a metric tree,
remarking that ``It is rare for a pair of metric spaces [\ldots] to have the isometric extension property.''
As a concrete example, the spaces $(\X,\rho)=(\R^3,\nrm{\cdot}_1)$
and
$(\Y,\ell)=(\R^2,\nrm{\cdot}_2)$
fail to have this property
\citep[Counterexample~2.4]{naor15}.

\paragraph{Open problem.}
The bounded in expectation (BIE)
condition on $\Y$ is a natural
generalization of the real-valued
variant that $\E|Y|<\infty$
(or, more generally, $\E|Y|^p<\infty$
if $L_p$ risk is being considered).
These conditions, while sufficient
for Bayes consistency, are clearly not always necessary.
Consider, for example,
$\X=\Y=\R$, 
endowed with the standard metric,
where the distribution 
$\bmu$
is such that
the $\X$-marginal is
Cauchy (i.e., has
density $f(x)=[\pi(1+x^2)]\inv$) and $X=Y$ almost surely.
In this case,
$\E|Y|=\infty$
and the more general
BIE condition also fails.
Yet the identity predictor $h(x)=x$ achieves the Bayes-optimal risk of $0$,
and various simple learning algorithms, including linear regression, achieve
Bayes consistency (we conjecture that
\medoidnet\
does as well).
Problem: formulate a necessary and sufficient condition
on the metric spaces
$(\X,\rho)$ and $(\Y,\ell)$,
and the joint distribution
$\bmu$ such that \medoidnet\
(or some other learning algorithm)
is strongly Bayes-consistent.
A natural and optimistic candidate is the condition
$R^*<\infty$.

\section{Definitions and notation}
\label{SEC:def-not}
For $n\in\N:=\set{1,2,\ldots}$, define $[n] := \{1,\ldots,n\}$;
for any set $\mathcal{Z}$, we write $\mathcal{Z}^+:=\bigcup_{n=1}^\infty\mathcal{Z}^n$
and $\mathcal{Z}^{\le k}:=\bigcup_{n=1}^k\mathcal{Z}^n$,
where $|z|$ denotes the sequence length.
For $A\in\mathcal{Z}^+$, we write $B\subset A$ to denote the subsequence relation.
Our instance and label spaces are the metric spaces $(\X,\rho)$
and $(\Y,\ell)$, respectively,
whose product Borel $\sigma$-algebra
is equipped with the probability
measure $\bmu$,
whose $\X$-marginal will be denoted
by $\mu$ and $\Y$-marginal by $\mu_\Y$.
We say that
$\Y$ is {\em bounded in expectation} (BIE)
if 
$\E_{(X,Y)\sim\bmu}\ell(y_0,Y)<\infty$
for some $y_0\in Y$.
Some of our results will also hold for countable $\Y$
equipped with an
arbitrary (possibly non-metric)
loss function $\ell:\Y\times\Y\to[0,\infty)$.
We denote set cardinalities by $|\Y|$
and the diameter by
\beqn
\label{eq:diam}
\nrm{\Y}\equiv
\diam(\Y):=\sup_{y,y'\in\Y}\ell(y,y');
\eeqn
the latter is also meaningful
when $\ell$ is not a metric.
For $x \in \X$ and $r > 0$, $B_{r}(x)$ 
denotes
the open ball of radius $r$ about $x$;
an analogous definition holds when $(\Y,\ell)$ is a metric.
Unless specified otherwise,
$S_n=(X_i,Y_i)_{i\in[n]}$
is always sampled iid from $\bmu$.
To any measurable mapping $f:\X\to\Y$
we associate the (true) risk
$\risk(f):=\E_{(X,Y)\sim\bmu}\ell(f(X),Y)$
and the empirical risk
$
\erisk
:\Y^\X
\times
(\X\times\Y)^+
\to
\R
$
by
\beqn
\erisk(f;S):=
|S|\inv\sum_{(x,y)\in S}\ell(f(x),y)
.
\eeqn
The Bayes-optimal risk is $\risk^*:=\inf_f\risk(f)$,
where the infimum is over all measurable $f:\X\to\Y$.

We implicitly assume the existence of fixed
measurable total orders on $\X$
and on $\Y$, whose existence is guaranteed by
\citet[Proposition D.1]{hksw21},
and refer to these orderings as {\em lexicographic}.
For $A \subseteq \X$, denote its $\g$-envelope by $\UB_\g(A)$ $:= \cup_{x\in A} B_\g(x)$ and consider the $\g$-{\em missing mass} of $S_n$, defined as the following random variable:
\beqn
\label{eq:Lg}
\missmass_\g(S_n) := \mu( \X \setminus \textup{\UB}_\g(S_n)).
\eeqn

As in \citet{hksw21},
we denote,
for any
labeled sequence $S=(x_i,y_i)_{i=1}^n\in(\X\times\Y)^n$, 
and any $x\in\X$, 
the nearest neighbor of $x$ with respect to $S$
and its label by
$X_{\nn}(x,S)$ and $Y_{\nn}(x,S)$,
respectively:
\beq
(X_{\nn}(x,S),Y_{\nn}(x,S)) := 
\argmin_{(x_i,y_i)\in S} 
\dist(x, x_i),
\eeq
where ties are broken lexicographically.
The $1$-NN predictor induced by $S$ is defined as $h_{S}( x ) := Y_{\nn}(x,S)$.
For any $m \in \nats$, any sequence $\bm X = \{x_1,\ldots,x_{m}\} \in \X^{m}$ induces a {\em Voronoi partition} of $\X$, $\Vor(\bm X) := \{V_1(\bm X),\dots, V_{m}(\bm X)\}$, where each Voronoi cell is
\beq
V_i(\bm X) := \set{x \in \X : i = \argmin_{1\le j\le m} \rho(x,x_j) },
\eeq
again breaking ties lexicographically.
In particular, for $\bm X = \{ X_i : (X_i,Y_i) \in S \}$, we have $h_S(x) = Y_i$ for all $x\in V_i(\bm X)$.
A $1$-NN algorithm is a mapping from an i.i.d.~labeled sample $S_n \sim \bmu^n$ to a labeled set $\tS_n \subseteq \X \times \Y$, yielding the $1$-NN predictor $h_{\tS_n}$. %
For $A\subseteq\X$ and $\g>0$, a $\g$-{\em net} of $A$ is any {\em maximal set} $B\subseteq A$ in which all interpoint distances are at least $\g$.
For a partition $\cal{A}$
of $B\subseteq\X$,
we write
$\nrm{\cal{A}}
:=\sup_{A\in\cal{A}}\nrm{A}
$
(again, as in \eqref{diam}, $\nrm{A}:=\diam A$).
Given a labeled set $S_n = (x_i,y_i)_{i \in [n]}$, $d \in [n]$, and any $\bm i = \{i_1,\ldots,i_d\} \in [n]^d$, denote the sub-sample of $S_n$ indexed by $\bm i$ by $S_n(\bm i) := \{(x_{i_1}, y_{i_1}), \dots,(x_{i_d}, y_{i_d})\}$.
Similarly, for a vector $\tby = \{y'_1,\ldots,y'_d\} \in \Y^d$, define $S_n{(\bm i, \tby)} := \{(x_{i_1}, y'_{1}), \dots,(x_{i_d}, y'_{d})\}$, namely the sub-sample of $S_n$ as determined by $\bm i$ where the labels are replaced with $\tby$.
Lastly, for $\bm i,\bm j \in [n]^d$, we denote $S_n(\bm i; \bm j) := \{(x_{i_1}, y_{j_1}), \dots,(x_{i_d}, y_{j_d})\}.$

We use standard order-of-magnitude notation throughout the paper;
thus, for $f,g:\N\to[0,\infty)$ we write $f(n)\in O(g(n))$ to mean $\limsup_{n\to\infty} f(n)/g(n)$ $<\infty$ and $f(n)\in o(g(n))$ to mean $\limsup_{n\to\infty} f(n)/g(n)=0$.
Likewise, $f(n)\in\Omega(g(n))$ means that $g(n)\in O(f(n))$.
In accordance with common convention, we often use the less precise notation $f(n)=O(g(n))$, etc.

We say that a metric space $(\X,\rho)$ is {\em separable} if it contains a dense countable set.
A metric probability space $(\X,\rho,\mu)$ is separable if there is a measurable $\X'\subseteq\X$ with $\mu(\X')=1$ such that $(\X',\rho)$ is separable.

A {\em sample compression scheme} $ (\compfunc, \reconfunc) $ of size at most $k$ using $b$ bits of side-information consists of a \emph{compression function} and a \emph{reconstruction function}. The \emph{compression function}  $\compfunc$ maps every finite sample set to $b$ bits plus a \emph{compression set}, which is a subset of at most $k$ labeled examples.
\beq 
\compfunc: 
\left(\X \times \Y\right)^+ \to  
\left(\X \times \Y\right)^{\le k} \times \left\{0, 1\right\}^b 
.
\eeq
The \emph{reconstruction function} $\reconfunc$ maps every possible compression set and $b$ bits to a hypothesis:
\beq
    \reconfunc: 
\left(\X \times \Y\right)^{\le k} \times \left\{0, 1\right\}^b     
\to
    \Y^\X
    .
\eeq

\section{Semi-stable compression}
In this section, we expand the definition of {\em stable compression}
and present our results for this notion.
First, we split the compression function $\compfunc$ into its two components.
For $S\in(\X\times\Y)^+$,
we write
$\compfunc(S)=(\compfunccs(S),\compfuncsi(S))\in (\X\times\Y)^{\le k}\times\set{0,1}^b$;
these are the {\em compression set} and the {\em side information}.
We say that $(\compfunc,\reconfunc)$ is {\em semi-stable} if the $\compfunccs$ component
is stable in the sense of \citet{BousquetHMZ20}:
whenever $\compfunccs(S)\subseteq S'\subseteq S$, we have that
\beq
\reconfunc(\compfunccs(S'), \compfuncsi(S)) = 
\reconfunc(\compfunccs(S), \compfuncsi(S)) 
= \reconfunc(\compfunc(S)).
\eeq
We denote by $|\compfunccs(\cdot)|$ and $|\compfuncsi(\cdot)|$ the sizes of the compression set and side information (in bits), respectively.

\begin{theorem}[proof deferred to Section %
\ref{ap:agnosemitstablewithinterpol}]
\label{thm:agnosemitstablewithinterpol}
\global\def\agnosemitstablewithinterpol{
Suppose that
$\X$ is an instance space and $\Y$ a label space with a loss function $\ell:\Y\times\Y\to[0,L]$,
and
$(\compfunc, \reconfunc)$ is semi-stable compression scheme.
For any distribution $\bmu$ over $\X \times \Y$, any $n \in \N$, and any $\delta \in (0,1)$, for $S_n \sim \bmu^n$
we have that
\begin{align*}
\risk(\reconfunc(\compfunc(S_n))) - \erisk(\reconfunc(\compfunc(S_n)); S_n) &\le 
\left(20\sqrt{\frac{|\compfunccs(S_n)|}{n}} + 20\sqrt{\frac{|\compfuncsi(S_n)|}{n}} + 15\sqrt{\frac{\ln(\frac{4e^2}{\delta})}{n}} \right)\erisk(\reconfunc(\compfunc(S_n)); S_n)
\\
& + (6L+18)\frac{|\compfunccs(S_n)|}{n}+8L\sqrt{\frac{|\compfunccs(S_n)|}{n}} + (2L+12)\frac{|\compfuncsi(S_n)|}{n} 
\\
&+ 7L\sqrt{\frac{|\compfuncsi(S_n)|}{n}}  + (3L+10)\frac{\ln(\frac{4e^2}{\delta})}{n} + 6L\sqrt{\frac{\ln(\frac{4e^2}{\delta})}{n}} 
\end{align*}
holds with probability at least $1 - \delta$.
}
\agnosemitstablewithinterpol
\end{theorem}

\subsection*{$(\a,\fns, \bits)$-semi-stable-compression}
Let $\X$, $\Y$, $\ell$, and $S_n$ be as in the statement of \thmref{agnosemitstablewithinterpol}.
For $\fns \leq n$, $\bits \in \mathbb{N}$, and $\a\ge0$, 
we say that $(\tS_n, h_{\tS_n})$ is an \emph{$(\a,\fns, \bits)$-semi-stable-compression} of $S_n$ if there exist $\bm i \in [n]^{\fns}$ and $\bY \in \Y^{\fns}$ such that:
\begin{enumerate}
\item $h_{\tS_n}$ and $\tS_n = S_n(\bm i, \bY)$ are a result of a semi-stable compression scheme of size $\fns$ with at most $\bits$ bits of side information. 
Thus,
$\tS_n=\compfuncsi(S_n)$
and
$h_{\tS_n}=\reconfunc(\compfunccs(S), \compfuncsi(S))$.
\item $
\erisk(h_{\tS_n};S_n)
\leq \alpha$.
\end{enumerate}
\begin{lemma}[proof in Section~\ref{sec:pfQbound}]
\label{lem:Qbound}
\def\loclabel{{\label{eq:Qbound}}}
\global\def\Qboundlem{
Let $\X$, $\Y$, $\ell$, $L$, and $S_n$ be as in 
\thmref{agnosemitstablewithinterpol}.
For $\fns \leq n$, define
\begin{align}
\loclabel
Q(n, \alpha, k, b, \delta, L) := \Qbound &:= \left(20\sqrt{\frac{k}{n}} + 20\sqrt{\frac{b}{n}} + 15\sqrt{\frac{\ln(\frac{4e^2}{\delta})}{n}} + 1 \right) \alpha \\
\nonumber
&\quad + (6L+18)\frac{k}{n}+8L\sqrt{\frac{k}{n}} + (2L+12)\frac{b}{n} + 7L\sqrt{\frac{b}{n}} \\
\nonumber
&\quad + (3L+10)\frac{\ln(\frac{4e^2}{\delta})}{n} + 6L\sqrt{\frac{\ln(\frac{4e^2}{\delta})}{n}}
.
\end{align}

Then the function $Q$ satisfies the following properties:
\begin{enumerate}[label=\textbf{\bQ{}\arabic*.},leftmargin=*]
\item[\bQ{1}.]  For any $n\in\N$ and $\delta \in (0,1)$, with probability at least $1 - \delta$ over $S_n \sim \bmu^n$, for all $\a \in [0,L]$, $\fns \in [n]$, $\bits \in \mathbb{N}$:
If $(\tS_n, h_{\tS_n})$ is an $(\a, \fns, \bits)$-semi-stable-compression of $S_n$, then 
$$
\risk(h_{\tS_n}) \leq \Qbound.
$$
\item[\bQ{2}.] For any fixed $n \in \N$ and $\delta \in (0,1)$, $Q$ is monotonically increasing in $\a$ and in $\fns$.

\item[\Qthreeb{\bf}.] There is a sequence $\{\delta_n\}_{n = 1}^\infty$, $\delta_n \in (0,1)$ such that $\sum_{n=1}^\infty \delta_n < \infty$, and for any $\fns_n \in o(n)$ we have that
\beq
\lim_{n\rightarrow \infty} \sup_{\a\in [0,L]} (Q_n(\a, \fns_n, \SIsize,\delta_n, L) - \a) = 0.
\eeq
\end{enumerate}
}
\Qboundlem
\end{lemma}

\section{Metric approximations}
Our proof technique involves performing several distinct truncations, approximating
a potentially unbounded quantity by a finite one. In this section, we adapt a variant
of this method from \citet{hksw21} for the 0-1 loss to arbitrary bounded losses.
Here,
$(\X,\rho)$ is assumed to be a separable metric space, and $\Y$ a \emph{countable} label space with a loss function $\ell
:\Y^2\to[0,L]
$.
Let $\SP = \{\sp_1,\dots\}$ be a countable partition of $\X$, and define the function $I_\SP: \X \to \SP$ such that $I_\SP(x)$ is the unique $\sp\in\SP$ for which $x\in\sp$.
For any measurable set $\emptyset\neq E\subseteq\X$ define the true medoid label $y^*(E)$ by 
\beqn
\label{eq:bar y(E)}
y^*(E)
= \argmin_{y\in \Y} \int_{X\in E} \ell(y,Y)\mathd\bmu,
\eeqn
where ties are broken lexicographically according to fixed total order on $\Y$. 
Given $\SP$ and a measurable set $\mss\subseteq \X$, define the true medoid predictor $h_{\SP,\mss}^*:\X\to\Y$ given by
\beqn
\label{eq:Strue2}
h_{\SP,\mss}^*(x) = y^*(I_\SP(x) \cap\mss).
\eeqn
\begin{lemma}[proof in Section~\ref{app:richness}]
\label{lem:richness}
\global\def\richnesslem{
Let $\bmu$ be a probability measure on $\X\times\Y$
with $\X$-marginal $\mu$, where $\X$ is a metric probability space, and $\Y$ a \emph{countable} label space 
with a 
loss function $\ell$ such that $L:=\nrm{\Y} < \infty$.
For any $\nu>0$, there exists a diameter $\beta=\beta(\nu)>0$ such that for any countable measurable partition $\SP = \{\sp_1,\dots\}$ of $\X$ and any measurable set $\mss\subseteq \X$ satisfying 
\begin{itemize}
\item[\myi] $\mu(\X\setminus\mss) \leq \nu$
\item[\myii] $
\sup_{\sp\in\SP}
\nrm{\sp \cap\mss}\leq \beta$,
\end{itemize}
the true medoid predictor $h_{\SP,\mss}^*$ defined in (\ref{eq:Strue2}) satisfies
\beq
\risk(h_{\SP,\mss}^*) \leq R^* + 9L\nu.
\eeq
}
\richnesslem
\end{lemma}
The proof of \lemref{richness} 
is similar to that of \citet[Lemma 3.6]{hksw21}, with novel arguments 
to handle the general loss function setting.
Next, we state 
two results from
\citet{hksw21}:
\begin{lemma}[variant of Lemma 3.7, \citet{hksw21}]
\label{lem:sublinear_comp}
\global\def\sublinearcomplem{
Let $(\X,\rho,\mu)$ be a separable metric probability space.
For $S_n\sim\mu^n$, let $\gnet(\g)$ be any $\g$-net of $S_n$.
Then, for any $\g>0$, there exists a function $\Ng:\N\to\R_+$ in $o(n)$ such that
$\displaystyle
\P\left[
\sup_{\g\text{-}\mathrm{nets}\,\,\gnet(\g)}
|\gnet(\g)| \geq \Ng(n)
\right]
\leq 
1/n^2.
$
}
\def\loclabel{\nonumber}
\sublinearcomplem
\end{lemma}

\begin{lemma}[Lemma 3.8, \citet{hksw21}]
\label{lem:missing_mass}
 Let $(\X,\rho,\mu)$ be a separable metric probability space, $\g>0$ be fixed, and the $\g$-missing mass $\missmass_\g$ defined as in (\ref{eq:Lg}).
Then there exists a function $u_\g:\N\to\R_+$ in $o(1)$, such 
that
$
\P\left[
\missmass_\g(S_n) \geq u_\g(n) + t
\right]
\leq 
\exp\left(-  n t^2\right)
$
for $S_n \sim\mu^n$ and $t>0$.
\end{lemma}

\section{
Algorithms and analysis: finite $\Y$
}
\label{SEC:COMPRESSION_SCHEME}

In this section,
we give the most basic version of our algorithm,
denoted \finmednet,
for the case where
$(\X,\rho)$ is a separable metric
and $\Y$ is a {\em finite}
set
equipped with an {\em arbitrary}
(not necessarily metric)
loss function $\ell:\Y\times\Y\to\R_+$.
This rudimentary setting provides
the
basis for extension to more general settings, in the sequel.

The input is the sample $S_n$; the set of instances in the sample is denoted by $\bm X_n = \{X_1,\ldots,X_n\}$. The algorithm defines a set $\Gamma$ of all
$n\choose2$
scales $\g > 0$ which are interpoint distances in $\bm X_n$, and the additional scale $\g = \infty$. 
For each scale in $\Gamma$, the algorithm constructs a $\g$-net of $\bm X_n$.
Denote the constructed $\g$-net by 
\beqn
\label{eq:rns-def}
\bm X(\g) := \{X_{i_1},\ldots,X_{i_{\rns}}\},
\eeqn where 
\beqn
\label{eq:rns}
\rns\equiv\rns_n(\g) :=|\gnet(\g)|
\eeqn
denotes its size and $\bm i \equiv \bm i(\g) := \{i_1,\ldots,i_{\rns}\} \in [n]^{\rns}$ denotes the indices selected from $S_n$ for this $\g$-net.  

For each $\g$-net, \algref{simple} finds the \emph{empirical medoid labels}
in the Voronoi cells defined by the partition $\Vor(\bm X(\g)) = \{V_1(\bm X(\g)),\ldots,V_{\rns}(\bm X(\g))\}$.
These labels are denoted by $\tbY(\g)\in \Y^{\rns}$.
Formally, for $i \in [\rns]$,
\beqn
\label{eq:maj} 
\tY_i(\g) := \argmin_{y \in \Y} {\sum_{j\in[n]: X_{j}\in V_{i}} \ell(y,Y_{j}) }.
\eeqn
As always,
ties are broken 
lexicographically.
The 
output of \finmednet\
is a labeled set $\tS_n(\g) := S_n(\bm i (\g), \tbY(\g))$ for every 
candidate
scale $\g \in \Gamma$.
The algorithm then selects a single scale $\g^*\equiv\g_n^*$ from $\Gamma$, and outputs the hypothesis that it induces, $h_{\tS_n(\g^*)}$. The choice of $\g^*$ is executed by minimizing a generalization error bound, denoted $Q$, which upper-bounds $\risk(h_{\tS_n(\g)})$ with high probability.

\RestyleAlgo{ruled}
\SetKwInOut{Assumptions}{Assumptions}
\SetKwInOut{Input}{Input}\SetKwInOut{Output}{Output}
\begin{algorithm}
\caption{\finmednet 
}
\label{alg:simple} 
\BlankLine
\Assumptions{$(\X,\rho)$ is a separable metric space, $\Y$ a \emph{finite} label space with a loss function $\ell$. Define 
$L:=\nrm{\Y}
=\max_{y,y'\in\Y}\ell(y,y')
$ and $b:=\log_{2}|\Y|$.}
\Input{Sample $S_n = (X_i, Y_i)_{i\in[n]}$, confidence $\delta_n \in (0,1)$}
\Output{predictor $h:\X\to\Y$}
\BlankLine
Let $\Gamma\gets(\set{\rho(X_i,X_j) : i,j \in [n]} \cup \{\infty\}) \setminus \{0\}$\;

\For{$\g \in \Gamma$}{

    Let $\bm X(\g)$ be a $\g$-net of $\{X_1,\ldots,X_n\}$\;
    
    Let $\rns_n(\g) \gets |\bm X(\g)|$\;

    For each $i \in [\rns_n(\g)]$, let $\tY_i(\g)$ be the \emph{empirical medoid label} of $V_i(\bm X(\g))$ as in \eqref{maj}\;
    
    Set $\tS_n(\g) \gets (\bm X(\g), \tbY(\g))$\;
    
    Set $h_{\tS_n(\g)} := x \mapsto Y_{\nn}(x,\tS_n(\g))$.
    
    Set $\a_n(\g) \gets \erisk(h_{\tS_n(\g)}; S_n)$\;
}
Find $\g^*_n \in  \argmin_{\g\in\Gamma} Q_{n}(\a_n(\g), \rns_n(\g), b, \delta, L)$, where $Q_n$ is defined in \eqref{Qbound}\;

Set $\tS_n \gets \tS_n({\g^*_n})$\;

\Return $h=h_{\tS_n}$\;
\end{algorithm}

\subsection*{Bayes Consistency of \finmednet}
The Bayes consistency result of
\citet{hksw21} was for the $0$-$1$ loss. Their approach was also compression-based, but
did not leverage the stability property, had no need for side information,
and did not have to truncate potentially unbounded losses.
Our main technical innovation was constructing 
\medoidnet\
(formally defined in Section~\ref{sec:separable})
as a semi-stable compression scheme with side-information, and then invoking it with an appropriate truncation schedule
for infinite and unbounded $\Y$. 

The first order of business is to verify that \finmednet\
indeed furnishes
a
semi-stable compression scheme
for any fixed $\g$:

\begin{lemma}[proof in Section~\ref{ap:semistablecompressionscheme}]
\label{lem:semistablecompressionscheme}
\global\def\semistablecompressionscheme{
Let $(\X,\rho)$ be a separable metric space, and
$\Y$ a \emph{finite} label space with a loss function $\ell$.
For any fixed scale $\gamma \in \Gamma$, the procedure 
in Algorithm~\ref{alg:simple} 
generating $h_{\tS_n(\g)}$ is a semi-stable compression scheme.
}
\semistablecompressionscheme
\end{lemma}
The following key technical lemma 
is a generalization of \citet[Lemma 3.5]{hksw21} from 0-1 loss to the general loss setting.
\begin{lemma}[proof in Section~\ref{ap:boundpd}]
\label{lem:boundpd}
\global\def\boundpdlem{
Let $\bmu$ be a probability measure on $\X\times\Y$, where $\X$ is a metric probability space, and $\Y$ a \emph{countable} label space endowed with a loss function $\ell
\le L<\infty$.
Let \(\Ng\) as 
in
\lemref{sublinear_comp}.
Then there exist functions $\eps \mapsto \g(\eps)$ and $\eps \mapsto 
\g :=
\nu(\eps) \in (0, \frac{\eps}{176L})$ such that for each $\eps, \bits > 0$ there is an $N_0(\nu(\eps), \SIsize, \delta_n, \Ng)$ such that for all $n \geq N_0$, and all $d \in [\Ng(n)]$, 
\begin{align*}
p_d &:= \P
\Big[
Q_n(\a_n(\g),\rns_n(\g),\SIsize, \delta_n, \Ln) > R^* \!+ \eps
\;\wedge\; \missmass_\g(S_n) \!\leq\! \frac{\eps}{18L}
\;\wedge\; \rns_n(\g)\!=\!d\Big] 
\leq e^{-\frac{n\eps^2}{32}} + e^{-\frac{1}{2}n\nu^{2}}\!,
\end{align*}
where 
$\rns_n(\g)$ is defined 
in \eqref{rns}.
}
\boundpdlem
\end{lemma}
The main result of this section is
\begin{theorem}
\label{thm:comp-consist}
Let $(\X,\rho)$ be a separable metric space, and $\Y$ a \emph{finite} label space with a loss function $\ell$.
Then there exists
a choice of $\delta_{n\in\N}$
such that
the sequence of
hypotheses
$h_{n}$ 
computed
by 
$\finmednet(S_n,\delta_n)$ is strongly \Bcstn:
$
\P[\lim_{n \rightarrow \infty} \risk(h_{n}) = R^*] = 1. 
$
\end{theorem}

\bepf
Recall that
$L:=\nrm{\Y}
=\max_{y,y'\in\Y}\ell(y,y')
$ and 
let
$b:=\log_{2}|\Y|$.
Let $Q$ be the generalization bound 
in 
\eqref{Qbound}
and set the input confidence $\delta$ for input size $n$ to $\delta_n$ as stipulated by \Qthreeb.

Given a sample $S_n\sim \bmu^n$, we abbreviate the optimal empirical error $\a_n^*=\a(\g^*_n)$ and the optimal compression size $\rns_n^*=\rns(\g^*_n)$ as computed by \algref{simple}.
By \lemref{semistablecompressionscheme}, 
the labeled set $\tS_n(\g_n^*)$ computed by \algref{simple} is an $(\a_n^*, \rns_n^*, \SIsize)$-semi-stable compression of the sample $S_n$.
For brevity we denote
$
Q_n(\alpha,\fns) := Q_n(\alpha,\fns,\SIsize, \delta_n, \Ln).
$
To prove 
the
Theorem,
we first follow the standard technique, used also in \citet{hksw21}, of decomposing the excess 
risk
into two terms:
\beq
\risk(h_{\tS_n(\g^*_n)}) - R^*  
&= &
\big(\risk(h_{\tS_n(\g^*_n)}) - Q_n(\a_n^*,\rns_n^*) \big)
+
\big(Q_n(\a_n^*,\rns_n^*) - R^*\big)
=:
T\subI(n) + T\subII(n)
\eeq
and arguing
that each term decays to zero almost surely.
For 
$T\subI(n)$ we have, similarly to \citeauthor{hksw21}, that Property \bQ1\ from \lemref{Qbound} implies that for any $n> 0$,
\beqn
\label{eq:termI_bound}
\P\!
\left[
\risk(h_{\tS_n(\g^*_n)}) - Q_n(\a_n^*,\rns_n^*) > 0
\right] 
\leq \delta_n.
\eeqn
Applying Borel-Cantelli
to the fact that $\sum \delta_n <\infty$ yields
$\limsup_{n\to\infty} T\subI(n) \leq 0$ almost surely.
The main departure from the proof in \citeauthor{hksw21}
is in
establishing  $\limsup_{n\to\infty} T\subII(n) \leq 0$ almost surely. 
We will argue that there exist $N = N(\eps) > 0$, $\g = \g(\eps) > 0$, $\nu = \nu(\eps) > 0$, and universal constants $c,C>0$ such that $\forall n \geq N$, 
\beqn
\label{eq:termII_bound_fixed_g}
\P[Q_n(\a_n(\g), \rns_n(\g))>R^* + \eps]
\leq
C n e^{-c n\eps^2} + n e^{-\nu^{2} n/2} + 1/n^2.
\eeqn
For any $\g > 0$ (even if $\g \notin \Gamma$), \algref{simple} finds a $\g_n^*$ such that
\beq
Q_n(\a_n^*, \rns_n^*) &=& \min_{\g' \in \Gamma} Q_n(\a_n(\g'),\rns_n(\g'))
\,\leq\, Q_n(\a_n(\g), \rns_n(\g)).
\eeq
The bound in (\ref{eq:termII_bound_fixed_g}) thus implies that $\forall n \geq N$,
\beqn
\label{eq:termII_bound}
\P[Q_n(\a_n^*, \rns_n^*) > R^* + \eps] 
\leq  C n e^{-cn\eps^2} + n e^{-\nu^{2} n/2} + 1/n^2.
\eeqn
By the Borel-Cantelli lemma, this implies that almost surely,
$
\limsup_{n\rightarrow \infty} T\subII(n) = \limsup_{n\rightarrow \infty} (Q_n(\a_n^*, \rns_n^*) - R^*) \leq 0.
$
Since $\forall n, T\subI(n) + T\subII(n) \geq 0$, this implies $\lim_{n\to\infty} T\subII(n) = 0$ almost surely, thus completing the proof.

It remains to prove (\ref{eq:termII_bound_fixed_g}),
the 0-1 loss
analog of 
which
was proved in \citeauthor[Eq. (3.4)]{hksw21}.
That argument does not hold for general losses, and we present the
novel argument below.
We bound the left-hand side of (\ref{eq:termII_bound_fixed_g}) using a function $n\mapsto \Ng(n)
\in
o(n)$, used to upper bound the compression size;
the latter is furnished by
\lemref{sublinear_comp}.
\beqn
\label{eq:split_miss}
&& \P[ Q_n(\a_n(\g),\rns_n(\g)) > R^* + \eps ]
\\\nonumber
& \leq &
\P
\left[ Q_n(\a_n(\g),\rns_n(\g)) > R^* + \eps 
\;\wedge\; \missmass_\g(S_n) \leq \frac{\eps}{18L}
\;\wedge\; \rns_n(\g) \leq \Ng(n)
\right]
\\
\nonumber
&& \,
+\, \P[ \missmass_\g(S_n) > \frac{\eps}{18L} ]
+ \P[ \rns_n(\g) > \Ng(n)] 
=:
P\subI+P\subII+P\subIII.
\eeqn
We estimate $P\subI$ via a union bound:
\begin{align*}
&\P
\left[ Q_n(\a_n(\g),\rns_n(\g)) > R^* + \eps 
\;\wedge\; \missmass_\g(S_n) \leq \frac{\eps}{18L}
\;\wedge\; \rns_n(\g) \leq \Ng(n)
\right]
\\
&\leq 
\sum_{d=1}^{\Ng(n)} \P
\Big[
Q_n(\a_n(\g),\rns_n(\g)) > R^* + \eps
\;\wedge\; \missmass_\g(S_n) \leq \frac{\eps}{18L}
\;\wedge\; \rns_n(\g)=d
\Big].
\end{align*}
Thus, it suffices to bound each term in the 
summation separately.
Applying \lemref{boundpd} and summing,
we have, for $n$ sufficiently large that $\Ng(n) \leq n$,
\beqn
\label{eq:first_term}
P\subI
\le
\sum_{d=1}^{\Ng(n)} p_d \;\leq\; \Ng(n)  (e^{-\frac{n\eps^2}{32}} + e^{-\frac{1}{2}n\nu^{2}}) \leq n (e^{-\frac{n\eps^2}{32}} + e^{-\frac{1}{2}n\nu^{2}}).
\eeqn
Now, using the function $\Ng$, we note that $P\subIII\le1/n^2$ thanks to \lemref{sublinear_comp}. A bound on $P\subII$, which bounds the $\gamma$-missing-mass $\missmass_\g(S_n)$, is furnished by \lemref{missing_mass}.
Taking $n$ sufficiently large so that $u_\g(n)$, as furnished by Lemma~\ref{lem:missing_mass}, satisfies $u_\g(n) \leq \eps/36L$, and invoking Lemma~\ref{lem:missing_mass} with $t=\eps/36L$, we have 
$
P\subII = \P[\missmass_\g(S_n) > \eps/18L] \leq e^{-\frac{n \eps^2}{1296L^{2}}}.
$
Plugging 
this,
\eqref{first_term},
and $P\subIII \leq 1/n^2$ into \eqref{split_miss} yields (\ref{eq:termII_bound_fixed_g}), which completes the proof.
\enpf
\section{Extensions}
\subsection{Countable $\Y$ with finite diameter:
\countable}
In this section we describe an extension of \finmednet, denoted \countable, which
is strongly 
Bayes-consistent
for countably
infinite $\Y$,
but still with a finite diameter.
A modification of \finmednet\ is required because
the latter uses a compression scheme with $b=\log_2|\Y|$ bits of side information.

Our variant is formally presented in \algref{simple2} and
operates as follows.
We fix in advance a specific sequence $\bits_n \in \mathbb{N}$, to be specified in the sequel.
The family of $\g$-nets over the input sample
is generated exactly as in \finmednet.
For each $\g$-net, \countable\ 
(presented in
Section~\ref{ap:alg-ctbl})
computes the \emph{truncated empirical medoid labels} in the Voronoi cells defined by the partition $\Vor(\bm X(\g)) = \{V_1(\bm X(\g)),\ldots,V_{\rns}(\bm X(\g))\}$.
These labels are denoted by $\tbY(\g)\in \first{\Y}{\bits_n}^{\rns}$.
Formally, for $i \in [\rns]$,
\beqn
\label{eq:maj2}
\tY_i
:= \argmin_{y \in \first{\Y}{\bits_n}
} {\sum_{j\in[n]: X_{j}\in V_{i}} \ell(y,Y_{j}) },
\eeqn
where
$\first{\Y}{b}
:=\set{y\in\Y':
\omega(y)\le 2^b
}
$
for
$b\in\N$,
for some
fixed canonical
injection $\omega:\Y\to\N$.
In words, $\first{\Y}{\bits_n}$
is a sample-dependent,
cardinality-based
truncation
of the label space.
Other than the truncation,
\countable\ 
behaves exactly as \finmednet.

\begin{theorem}[proof in Section~\ref{sec:pf-consist2}]
\label{thm:comp-consist2}
\global\def\compconsisttwothm{
Let $(\X,\rho)$ be a separable metric space, and $\Y$ a \emph{countable} label space with a loss function $\ell
\le
L
< \infty$.
Then there is
a choice of $\delta_{n\in\N}$
and truncation schedule $b_{n\in\N}$
such that
the sequence of
hypotheses
$h_{n}$ 
computed by
$\countable(S_n,\delta_n,\bits_n)$ is strongly \Bcstn:
$
\P[\lim_{n \rightarrow \infty} \risk(h_{n}) = R^*] = 1. 
$
}
\compconsisttwothm
\end{theorem}

\subsection{Countable metric space $(\Y, \ell)$ with unbounded diameter:
\ctblunbdd
}
\label{sec:unboundeddiam}

In this section,
we extend \countable\
to the case where $(\Y,\ell)$
is a countable metric space.
That is, 
the loss $\ell$
is now assumed to be a metric,
but the boundedness condition
$\nrm{\Y}<\infty$
is relaxed to
boundlessness-in-expectation (BIE):  $\E_{(X,Y)\sim\bmu}\ell(y_0,Y)<\infty$ for some $y_0\in\Y$.
Boundedness was used in the analysis of
\countable\ in order to invoke
a distribution-free concentration inequality
(Hoeffding's). The present extension,
denoted \ctblunbdd, invokes \countable\
as a subroutine with an appropriately
diameter-truncated label space.
The latter is defined as follows.
Fix a $y_0\in\Y$ that is a witness
of the BIE property\footnote{
\lemref{y0} shows that
if BIE holds then {\em every}
$y'\in\Y$ is such a witness,
and in particular, 
we may always choose
$y_0$ as the ``first''
element under the canonical ordering.
}. For $y\in\Y$ and $L>0$, define 
$\trunc{\Y}_L:=
B(y_0,L)
$
and
the {\em diameter-truncation} operation
\beqn
\label{eq:trunc}
y\wedge L :=
\argmin_{
\hat{y}\in\trunc{\Y}_L
}
 \ell(y,\hat{y})
.
\eeqn
In words, $y\wedge L$ is the closest $\hat{y}$ to $y$ in the $L$-ball about $y_0$.

\ctblunbdd\
is formally presented in \algref{simple3}
and operates as follows.
The cardinality- and diameter-truncation
schedules $b_{n\in\N}$ and $L_{n\in\N}$
are fixed in advance;
the former as any $\bits_n\in o(n)$
and the latter specified in the sequel.
Next, the labels $Y_i$ of the input sample are truncated to
$\trunc{Y_i} := Y_{i} \wedge L_n;$
this is a substantive difference from the cardinality-based
truncation in \countable, which does not modify the sample labels.

\RestyleAlgo{ruled}
\SetKwInOut{Assumptions}{Assumptions}
\SetKwInOut{Input}{Input}\SetKwInOut{Output}{Output}
\begin{algorithm}
\caption{\ctblunbdd}
\label{alg:simple3} 
\BlankLine
\Assumptions{$(\X,\rho)$ is a separable metric space, $(\Y, \ell)$ a BIE countable metric space
}
\Input{Sample $S_n = (X_i, Y_i)_{i\in[n]}$, $\delta_n \in (0,1)$,
$\bits_n\in\N$,
$L_n>0$
}
\Output{predictor $h:\X\to\Y$}
\BlankLine
Set $\trunc{S_n} := \{(X_i,\trunc{Y_i}): i \in [n] \}$,
where $\trunc{Y_i} := Y_{i} \wedge L_n$
\;

Set $h_{n} := \countable(\trunc{S_n}, \delta_n,\bits_n)$ in 
\emph{truncated} 
label space $
\trunc{\Y}_{L_n}
$
\\
\phantom{
Set $h_{n} := \countable
~~~~~~~
$}
(i.e.,
$\Y$ in \eqref{maj} 
is replaced with
$
\trunc{\Y}_{L_n}
$)
\;

\Return $h=h_{n}$
\end{algorithm}

\begin{theorem}[proof in Section~\ref{sec:pf-consist3}]
\label{thm:comp-consist3}
\global\def\compconsistthreethm{
Let $(\X,\rho)$ 
and
$(\Y, \ell)$
be
metric spaces, separable and countable, respectively,
equipped with a product distribution $\bmu$
such that
BIE holds for
$\Y$.
Then there is
a choice of $\delta_{n\in\N}$
and truncation schedules $b_{n\in\N}$, $L_{n\in\N}$
such that
the sequence of
hypotheses
$h_{n}$ 
computed by
$\ctblunbdd(S_n,\delta_n,\bits_n,L_n)$ is strongly \Bcstn:
$\displaystyle
\P[\lim_{n \rightarrow \infty} \risk(h_{n}) = R^*] = 1. 
$
}
\compconsistthreethm
\end{theorem}
The only remaining extension to render the proof of \thmref{main}
complete is from countable to {\em separable} $(\Y,\ell)$;
this straightforward step is carried out in Section~\ref{sec:separable}.

\appendix

\section{Deferred results}
\label{ap:deferred}

\subsection{Separable metric space $(\Y, \ell)$ with unbounded diameter}
\label{sec:separable}
The extension 
from countable to separable $(\Y,\ell)$
---
implemented by the final, subscript-free version of \medoidnet
---
is 
quite straightforward.
The approximation arguments we invoke
are standard, and hence we only give a sketch
of the proof.
In Section~\ref{sec:discr},
we give a countable discretization
$\Y_\eps\subseteq\Y$,
with a corresponding
discretized version
$\bmu_\eps$
of
$\bmu$
and
the induced Bayes-optimal risk $R^*_\eps$
on the discretized space.
\thmref{discr}
guarantees
that $R^*_\eps\to R^*$
as $\eps\to0$.

As discussed in the Introduction,
we
assume
access to
an oracle
that takes $\eps>0$ as input
and returns a (necessarily at most countable, due to separability)
$\eps$-net $\Y_\eps$ of $\Y$.
Given this oracle,
\medoidnet\ operates as follows.
First, a sequence $\eps_n\downarrow0$ is fixed.
For each $n\in\N$,
the sample $S_n$ is drawn
and the $\eps$-net $\Y_n:=\Y_{\eps_n}$
is constructed.
Next, each label $Y_i$ in $S_n$
is projected onto $\Y_n$
--- i.e., replaced by $Y_i'\in\Y_n$
that is closest to $Y_i$.
The resulting modified sample $S_n'$
is then fed into 
$\ctblunbdd$
with the additional arguments
$\delta_n,\bits_n,L_n$
as in \thmref{comp-consist3}.
The latter shows that
almost surely, the 
the constructed predictor's
risk 
minus
$R^*_{\eps_n}$
decays to zero.\footnote{
Formally,
\thmref{comp-consist3}
proves convergence on
a fixed label space $\Y$,
but a standard diagonal argument
lets us
apply it to the sequence $\Y_n$
and conclude the aforementioned claim.
}
This, coupled with
\thmref{discr},
implies
\thmref{main}:
\begin{theorem}
\label{thm:comp-consist4}
Let $(\X,\rho)$ and be $(\Y, \ell)$ separable metric spaces
equipped with a product distribution $\bmu$
such that
BIE holds for
$\Y$.
For any $\eps_n\downarrow0$,
let $\Y_n$ be a sequence of
$\eps_n$-nets as above.
Discretize each sample $S_n\sim\bmu^n$
to 
$S_n'$
with labels in $\Y_{n}$, as above.
Then there is
a choice of $\delta_{n\in\N}$
and truncation schedules $b_{n\in\N}$, $L_{n\in\N}$
such that
the sequence of
hypotheses
$h_{n}$ 
computed by
$\ctblunbdd(S_n',\delta_n,\bits_n,L_n)$ is strongly \Bcstn:
$
\P[\lim_{n \rightarrow \infty} \risk(h_{n}) = R^*] = 1. 
$
\end{theorem}

\clearpage 
\subsection{The \countable\ algorithm}
\label{ap:alg-ctbl}

\RestyleAlgo{ruled}
\SetKwInOut{Assumptions}{Assumptions}
\SetKwInOut{Input}{Input}\SetKwInOut{Output}{Output}
\begin{algorithm}
\caption{\countable}
\label{alg:simple2} 
\BlankLine
\Assumptions{Let $(\X,\rho)$ be a separable metric space, and $\Y$ a \emph{countable} label space with a loss function $\ell$ such that $L := \nrm{\Y} < \infty$.}
\Input{Sample $S_n = (X_i, Y_i)_{i\in[n]}$, confidence $\delta_n \in (0,1)$, side-information size $\bits_n
\in\N
$}
\Output{predictor $h:\X\to\Y$}
\BlankLine
Let $\Gamma\gets(\set{\rho(X_i,X_j) : i,j \in [n]} \cup \{\infty\}) \setminus \{0\}$\;

\For{$\g \in \Gamma$}{

    Let $\bm X(\g)$ be a $\g$-net of $\{X_1,\ldots,X_n\}$\;
    
    Let $\rns_n(\g) \gets |\bm X(\g)|$\;

    For each $i \in [\rns_n(\g)]$, let $\tY_i(\g)$ be the \emph{truncated empirical medoid label} of $V_i(\bm X(\g))$ as in \eqref{maj2} \;
    
    Set $\tS_n(\g) \gets (\bm X(\g), \tbY(\g))$\;
    
    Set $h_{\tS_n(\g)} := x \mapsto Y_{\nn}(x,\tS_n(\g))$.
    
    Set $\a_n(\g) \gets \erisk(h_{\tS_n(\g)}; S_n)$\;
}
Find $\g^*_n \in  \argmin_{\g\in\Gamma} Q_{n}(\a_n(\g), \rns_n(\g), \bits_n, \delta, L)$, where $Q_n$ is defined in \eqref{Qbound}\;

Set $\tS_n \gets \tS_n({\g^*_n})$\;

\Return $h=h_{\tS_n}$\;
\end{algorithm}

\section{Auxiliary Proofs}
\label{ap:proofs}

\subsection{Proof of \lemref{semistablecompressionscheme}} \label{ap:semistablecompressionscheme}
\begin{lemma*}
\semistablecompressionscheme
\end{lemma*}
\bepf
Fix a $\g \in \Gamma$. Define $b := \log_{2}|\Y|$
and the map $\bitenc:\Y\to\set{0,1}^\bits$
as one that converts the lexicographic
index of $y\in\Y$ to its unique $\bits$-bit
binary representation. Our compression function: 
\beq
\left(\X \times \Y\right)^+ 
\to  
\left(\X \times \Y\right)^+
\times  \{0, 1\}^{\SIsize}
.
\eeq
Recall our notation $\bi_\gamma$ as the $\gamma$-net indices calculated and selected for a sample $S_n \sim (\X \times \Y)^n$, and $\bY'(\bi_\gamma)$ the empirical medoid labels of a $\gamma$-net $\boldsymbol{X}(\bi_\gamma)$, as defined in \eqref{maj}.
Let
$\compfunc$ 
be such that
$\compfunc(S_n) = (\compfunccs(S_n), \compfuncsi(S_n))$,
where
\begin{align*}
    \compfunccs(S_n) &= S_n(\bi_{\gamma}) \in \left(\X \times \Y\right)^{|\bi_{\gamma^*}|} \\
    \compfuncsi(S_n) &= \left\{ \bitenc(Y') : Y' \in \bY'(\bi_{\gamma}) \right\}
    \in \left(\{0,1\}^{\SIsize}\right)^{|\bi_{\gamma^*}|}
    .
\end{align*}
In words,
$\compfunccs$ compresses the sample $S_n$ 
to a specific $\gamma$-net keeping original labels,
while $\compfuncsi$  calculates the respective empirical limited medoid labels of the resulting sub-sample.
As for the reconstruction function
\beq
  \reconfunc: 
\left(\X \times \Y\right)^+
  \times  \{0, 1\}^{\SIsize}
  \to 
\Y^\X,
\eeq
it is defined as
\beq
    \reconfunc\left(S_{n}(\bi_\gamma), \bY'(\bi_\gamma)\right) = h_{S_{n}\left(\bi_\gamma, \bY'(\bi_\gamma)\right)};
\eeq
in words,
we take the $1$-nearest-neighbor rule predictor of the sub-sample $S_n(\bi_\gamma)$ labeled by $\bY'(\bi_\gamma)$.

It remains to argue that our compression scheme is semi-stable.
Indeed,
since the scale $\g$ is fixed
and
the net is constructed in a deterministic fashion,
for any $S'$ 
satisfying
$\compfunccs(S_n) \subseteq S' \subseteq S_n$, 
the $\g$-net 
computed by the algorithm will be the same.
Therefore $\compfunccs(S')=\compfunccs(S_n)$ and the definition follows.
\enpf

\subsection{Proof of \lemref{Qbound}} \label{ap:Qbound}
\label{sec:pfQbound}
\begin{lemma*}
\def\loclabel{\nonumber}
\Qboundlem
\end{lemma*}

\bepf
Let $\X$ be an instance space, and $\Y$ a label space with a loss function $\ell$ such that $L:=\nrm{\Y} < \infty$.
Starting from \bQ1, let $(\tS_{n}, h_{\tS_n})$ be an $(\a, \fns, \bits)$-semi-stable-compression of $S_n$. Thus, 
$Q$ satisfies property $\bQ1$
by \thmref{agnosemitstablewithinterpol}.

Furthermore, property \bQ2 (monotonicity in $\a$ and in $\fns$) can also be easily verified from the  
definition in \eqref{Qbound}.

To 
establish
\Qthreeb, 
an inspection of
$Q_n(\a,\fns_n, \bits,\delta_n, L) - \a$
shows that since
$\fns_n \in o(n)$, 
only the terms containing
$ \frac{\ln(\frac{4e^2}{\delta_n})}{n} $ 
are not {\em obviously} decaying to zero.
To ensure the latter, any choice of
$\delta_n$
with
$-\log\delta_n \in o(n)$ suffices.
The additional constraint
$\delta_n$
must satisfy is
$\sum_{n=1}^\infty \delta_n < \infty$;
one such choice 
$\delta_n = e^{-\sqrt{n}}$.
\enpf

\subsection{Proof of \lemref{richness}}
\label{app:richness}
\begin{lemma*}
\richnesslem
\end{lemma*}

\bepf
We begin
similarly to 
the proof of
\citet[Lemma 3.6]{hksw21}.
Let $\eta_y: \X \to [0,1]$ be the conditional probability function for label $y\in\Y$,
\beq
\eta_y(x) = \P( Y = y \gn X = x),
\eeq
and let $ \zeta_y : \X \to [0,L] $ be the expected loss function for the label $y \in \Y$:
\beq
\zeta_y(x) = \E_\bmu\left[ \ell(y,Y) \gn X = x\right] = \int_{y'\in\mathcal{Y}}\ell(y,y')\eta_{y'}(x)\mathd\mu
\eeq
which is measurable by \citet[Corollary B.22]{MR1354146}.

Define $\teta_y: \X \to [0,1]$ as $\eta_y$'s conditional expectation function with respect to $(\SP,\mss)$: For $x$ such that 
$I_\SP(x) \cap\mss \neq \emptyset$,
\beq
\teta_y(x) = \P(Y = y \gn X \in I_\SP(x) \cap\mss)
= \frac{\int_{I_\SP(x) \cap\mss} \eta_y(z) \diff\mu(z) }{\mu(I_\SP(x) \cap\mss)}.
\eeq
Otherwise,
if
$I_\SP(x) \cap\mss = \emptyset$,
define $\teta_y(x) = \pred{y\text{ is lexicographically first}}$.
Note that $(\teta_y)_{y\in\Y}$ are piecewise constant on the cells of the restricted partition~$\SP \cap\mss$. Accordingly, define $\tilde{\zeta}_y \to [0,L]$:
\beq
\tilde{\zeta}_y(x) = \E_\bmu\left[ \ell(y,Y) \gn X \in I_\SP(x) \cap W\right]
= \sum_{y'\in\mathcal{Y}}\ell(y,y')\teta_{y'}(x).
\eeq
In the proof of
\citeauthor{hksw21},
there is no appearance
of 
$\zeta_y$
or
$\tilde\zeta_y$;
there, the much simpler
conditional {\em error probabilities}
suffice. Note likewise
that their local majority vote
classifier has been replaced
here by the local medoid.
By definition, the Bayes-optimal predictor $h^*$ and the true medoid predictor $h^*_{\SP,\mss}$ satisfy 
\beq
h^*(x) &=& \argmin_{y\in\Y} \zeta_y(x),
\\
h^*_{\SP,W}(x) &=& \argmin_{y\in\Y} \tilde{\zeta}_y(x).
\eeq
It follows that
\begin{align*}
\E_\bmu\left[ \ell(h^*_{\SP,W}(X),Y) \gn X = x \right]
&-
\E_\bmu\left[ \ell(h^*(X),Y) \gn X = x \right]\\
&= \zeta_{h^*_{\SP,W}(x)}(x) - \zeta_{h^*(x)}(x) 
\\ &=  \zeta_{h^*_{\SP,W}(x)}(x) - \tilde{\zeta}_{h^*_{\SP,W}(x)}(x) + \tilde{\zeta}_{h^*_{\SP,W}(x)}(x) - \tilde{\zeta}_{h^*(x)}(x) \\
&+ \tilde{\zeta}_{h^*(x)}(x) - \zeta_{h^*(x)}(x)
\\
&\leq \zeta_{h^*_{\SP,W}(x)}(x) - \tilde{\zeta}_{h^*_{\SP,W}(x)}(x) + \tilde{\zeta}_{h^*(x)}(x) - \zeta_{h^*(x)}(x) \\
&\leq 2\max_{y'\in\left\{h^*_{\SP,W}(x), h^*(x)\right\}} \left| \zeta_{y'}(x) - \tilde{\zeta}_{y'}(x) \right| \\
&= 2\max_{y'\in\left\{h^*_{\SP,W}(x), h^*(x)\right\}} \left| \sum_{y\in\mathcal{Y}}\ell(y',y)\left(\eta_{y}(x) - \teta_{y}(x)\right) \right| \\
&\leq 2L\left| \sum_{y\in\mathcal{Y}}\left(\eta_{y}(x) - \teta_{y}(x)\right) \right| \\
&\leq 2L\sum_{y\in\mathcal{Y}}\left|\eta_{y}(x) - \teta_{y}(x)\right|.
\end{align*}
By condition \myi\ in the lemma statement, $\mu(\X\setminus\mss)\leq \nu $.
Thus,
\beqn
\nonumber
\risk(h^*_{\SP,W}) - R^* & = &
\E_\bmu\left[ \ell(h^*_{\SP,W}(X),Y) \right] - \E_\bmu\left[ \ell(h^*(X),Y) \right]
\\
\nonumber
& = &
\int_{\X\setminus W} \left( \E_\bmu\left[ \ell(h^*_{\SP,W}(X),Y) \gn X = x \right] - \E_\bmu\left[ \ell(h^*(X),Y) \gn X = x \right] \right)\mathd\mu(x) \\
\nonumber
&& \quad + \int_{W} \left( \E_\bmu\left[ \ell(h^*_{\SP,W}(X),Y) \gn X = x \right] - \E_\bmu\left[ \ell(h^*(X),Y) \gn X = x \right] \right)\mathd\mu(x)
\\
\nonumber
& \leq & L \cdot \mu(\X\setminus W) + 2L\int_{W} \sum_{y\in\mathcal{Y}} \left|\eta_{y}(x) - \teta_{y}(x)\right| \mathd\mu(x)
\\
\nonumber
& \leq & L\nu + 2L\sum_{y\in\mathcal{Y}} \int_{W} \left|\eta_{y}(x) - \teta_{y}(x)\right| \mathd\mu(x).
\eeqn
Let $\Y_\nu \subseteq \Y$ be a finite set of labels such that $\P[Y \in \Y_\nu] \geq 1-\nu$. Then 
\beqn
\nonumber
\sum_{y\notin \Y_\nu} \int_{W} \left| \eta_{y}(x) - \tilde{\eta}_{y}(x) \right| \mathd\mu(x) & \le & \sum_{y\notin \Y_\nu} \int_{W} \eta_{y}(x) \mathd\mu(x) \\
\nonumber
& = & \sum_{y\notin \Y_\nu} \int_{W} \P_\bmu( Y = y \gn X = x) \mathd\mu(x) \\
\nonumber
& = & \sum_{y\notin \Y_\nu} \P_\bmu( Y = y , X \in W) \\
\nonumber
& \le & \sum_{y\notin \Y_\nu} \P(Y=y) \\
\nonumber
& = & \P(y \notin \Y_\nu) \le \nu .
\nonumber
\eeqn
We conclude:
\beqn
\label{eq:err_R_eta_teta}
\risk(h^*_{\SP,W}) - R^* \leq 3L\nu + 2L\sum_{y\in \Y_\nu}\int_{W} |\eta_y(x) - \teta_y(x)| \mathd\mu(x).
\eeqn
To bound the integrals in (\ref{eq:err_R_eta_teta}), we invoke a result from the proof of Lemma 3.6 in \citet{hksw21}, which showed that 
\beq
\sum_{y\in\Y_\nu}\int_{\mss} |\eta_y(x) - \teta_y(x)| \diff\mu(x)
\leq 
\sum_{y\in\Y_\nu} \frac{3\nu}{|\Y_\nu|} = 3\nu.
\eeq
Applying this bound to (\ref{eq:err_R_eta_teta}), we conclude 
$
\risk(h^*_{\SP,\mss}) - R^* \leq 9L\nu.
$
\enpf

\subsection{Proof of \lemref{boundpd}}
\label{ap:boundpd}

\begin{lemma*}
\boundpdlem
\end{lemma*}

\bepf
We begin the proof
similarly to
\citet[Lemma 3.5]{hksw21}
and then diverge in order to extend their 0-1
loss
to the general loss setting.
Let $\bm i = \bm i(\g) \in [n]^d$ be the set of indices in the net $\bm X = \bm X(\gamma)$ selected by the algorithm. %
Let ${\bm Y}^* \in \Y^d$ be the true medoid labels with respect to the restricted partition $\Vor(\bm X)\cap\UB_{2\g}(\bm X)$,
\beqn
\label{eq:Y_star}
({\bm Y}^*)_j = 
y^*(V_j \cap \UB_{2\g}(\bm X)),
\qquad j\in[d].
\eeqn
We pair $\bm X$ with the labels $\bm Y^*$ to obtain the labeled set
\beqn
\label{eq:Strue}
S_n(\bm i, *) := S_n(\bm i, {\bm Y}^*) = 
(\bm X,\bm Y^*)
\in (\X\times\Y)^d.
\eeqn
Note that conditioned on $\bm X$, $S_n(\bm i, *)$ does not depend on the rest of $S_n$. 

The induced $1$-NN predictor $h_{S_n(\bm i, *)}(x)$ can be expressed as 
$h_{\SP,\mss}^*(x) = y^*(I_\SP(x) \cap\mss)$
with $\SP = \Vor(\bm X)$
and $\mss=\UB_{2\g}(\bm X)$
(see \eqref{Strue2} for the definition of $h_{\SP,\mss}^*$). We now show that
\beqn
\label{eq:err_R}
\missmass_\g(\bm X_n) \leq \frac{\eps}{18L}
\quad \implies \quad
\risk(h_{S_n(\bm i, *)})\leq R^* + \eps/2,
\eeqn
by showing that under the assumption $\missmass_\g(\bm X_n) \leq \frac{\eps}{18L}$, the conditions of \lemref{richness} hold for $\Vor,\mss$ as defined above.
To this end, we bound the diameter of the partition $\Vor \cap\mss = \Vor \cap \UB_{2\g}(\bm X)$, and the measure of the missing mass $\mu(\X\setminus\mss) = \missmass_{2\g}(\bm X)$ under the assumption.

To bound the diameter of the partition $\Vor \cap \UB_{2\g}(\bm X)$, let $x \in V_j \cap  \UB_{2\g}(\bm X)$. Note that $V_j$ is the Voronoi cell centered at $x_{i_j} \in \bm X$. Then $\rho(x,x_{i_j}) = \min_{i \in \bm i} \rho(x,x_i)$ and, since $x \in \UB_{2\g}(\bm X)$, $\min_{i \in \bm i} \rho(x,x_i) \leq 2\gamma$. Thus,
\beq
\nrm{\SP \cap\mss} = \max_j\nrm{V_j \cap \UB_{2\g}(\bm X)}\leq 4\g.
\eeq
To bound $\missmass_{2\g}(\bm X)$ under the assumption $\missmass_\g(\bm X_n) \leq \frac{\eps}{18L}$,
observe that for all $z \in \UB_{\g}(\bm X_n)$,
there is some $i \in [n]$ such that $z \in B_\g(x_i)$.
For this $i$, there is some $j \in \bm i$ such that $x_i \in B_\g(x_j)$,
since $\bm X$ is a $\g$-net of $\bm X_n$. Therefore $z \in B_{2\g}(x_j)$.
Thus, $z \in \UB_{2\g}(\bm X)$. It follows that
$\UB_\g(\bm X_n) \subseteq \UB_{2\g}(\bm X)$, thus $\missmass_{2\g}(\bm X) \leq \missmass_\g(\bm X_n)$.
Under the assumption, we thus have $\missmass_{2\g}(\bm X) \leq \frac{\eps}{18L}$.
Hence, by the choice of $\g=\g(\eps)$ in the statement of the lemma, \lemref{richness} implies \eqref{err_R}.

To bound $Q_n(\a_n(\g), \rns_n(\g))$, we consider the relationship between the hypothetical true medoid predictor $h_{S_n(\bm i, *)}$ and the actual predictor returned by the algorithm, $h_{S_n(\bm i, \bm \tbY)}$.
Firstly, for any $\nu \in (0,1)$, there exists a finite $\Y'_\nu \subseteq \Y$ such that $\P\left[ Y \in \Y'_\nu\right] \ge 1 - \nu$. Therefore, since $\SIsize$ is non-decreasing in $n$, there exists an $N_{1}(\nu)$ large enough such that for any $n \ge N_{1}(\nu)$ $\Y'_\nu \subseteq 
\first{\Y}{\SIsize} 
$. Fix such a $\nu$ specifically such that $\nu < \frac{\eps}{176L}$. Thus we have that $\P\left[ Y \in \first{\Y}{\SIsize}\right] \ge 1 - \nu$.

For brevity we denote
\beq
Q_n(\alpha,\fns) := Q_n(\alpha,\fns,\SIsize, \delta_n, \Ln).
\eeq
Let us 
split our cases:
\beqn
\nonumber
p_d &=& \P\big[Q_n(\a_n(\g),\rns_n(\g))
>
R^* + \eps
\;\;\wedge\;\;
\missmass_\g(\bm X_n) \leq \frac{\eps}{18L}
\;\;\wedge\;\;
\rns_n(\g)=d \\
\nonumber
&& \;\;\;
\;\;\wedge\;\;
\erisk(h_{S_n(\bm i, \bm \tbY)}; S_n) \le \erisk(h_{S_n(\bm i, *)}; S_n) + 2\nu L
\big] +
\\
\nonumber
&& + \quad 
\P\big[Q_n(\a_n(\g),\rns_n(\g))
>
R^* + \eps
\;\;\wedge\;\;
\missmass_\g(\bm X_n) \leq \frac{\eps}{18L}
\;\;\wedge\;\;
\rns_n(\g)=d \\
\nonumber
&& \;\;\;
\;\;\wedge\;\;
\erisk(h_{S_n(\bm i, \bm \tbY)}; S_n) > \erisk(h_{S_n(\bm i, *)}; S_n) + 2\nu L
\big]
\\
&:=& 
(p_d)_1 + (p_d)_2.
\eeqn

Now, let $\bY^*$ be the best medoids possible for the sample, from the entire label space $\Y$:
\beq \erisk(h_{S_n(\bi, \bY^*)}; S_n) = \min_{\bY \in \Y^d} \erisk(h_{S_n(\bi, \bY)}; S_n). \eeq
This means:
\beq \erisk(h_{S_n(\bi, \bY^*)}; S_n) \le \erisk(h_{S_n(\bi, \bY')}; S_n)
\qquad\text{ and }\qquad
\erisk(h_{S_n(\bi, \bY^*)}; S_n) \le \erisk(h_{S_n(\bi, *)}; S_n)
.
\eeq
Specifically, we note that since
\beq \erisk(h_{S_n(\bi, \bY')}; S_n) = \min_{\bY \in \first{\Y}{\bits_n}^{d}} \erisk(h_{S_n(\bi, \bY)}; S_n)
,
\eeq
we have that
\begin{align*}
\erisk(h_{S_n(\bi, \bY')}; S_n) - \erisk(h_{S_n(\bi, \bY^*)}; S_n) &= \frac{1}{n}\sum_{(X,Y)\in S_n, Y\notin \first{\Y}{\bits_n}} \ell(h_{S_n(\bi, \bY')}(X),Y) - \ell(h_{S_n(\bi, \bY^*)}(X),Y) \\
&\le \frac{1}{n} L \left| \left\{ (X,Y)\in S_n, Y\notin \first{\Y}{\bits_n} \right\} \right| \\
\Rightarrow \E\left[ \erisk(h_{S_n(\bi, \bY')}; S_n) - \erisk(h_{S_n(\bi, \bY^*)}; S_n) \right] &\le \frac{1}{n}L\cdot \nu n = \nu L
.
\end{align*}
Hence, we observe:
\beq \E\left[ \erisk(h_{S_n(\bi, \bY')}; S_n) - \erisk(h_{S_n(\bi, \bY^*)}; S_n) \right] \le \nu L
\qquad
\text{ and }
\qquad
\E\left[ \erisk(h_{S_n(\bi, \bY^*)}; S_n) - \erisk(h_{S_n(\bi, *)}; S_n) \right] \le 0 , \eeq
whence
\beq \E\left[ \erisk(h_{S_n(\bi, \bY')}; S_n) - \erisk(h_{S_n(\bi, *)}; S_n) \right] \le \nu L .\eeq
Next, Hoeffding's inequality 
implies
that for any $t > 0$:
\beq \P\left( \erisk(h_{S_n(\bi, \bY')}; S_n) - \erisk(h_{S_n(\bi, *)}; S_n) > \nu L + t \right) < \exp\left( -\frac{nt^2}{2L^{2}} \right)
.
\eeq
Taking $t = \nu L$ we get:
\beq (p_d)_2 \le \P\left( \erisk(h_{S_n(\bi, \bY')}; S_n) - \erisk(h_{S_n(\bi, *)}; S_n) > 2\nu L \right) < \exp\left( -\frac{1}{2}n\nu^2 \right).\eeq
Next,
we examine the case $  \erisk(h_{S_n(\bi, \bY')}; S_n) - \erisk(h_{S_n(\bi, *)}; S_n) \le 2\nu L $, and use the monotonicity Property \bQ2 of $Q$:
\begin{align*}
    Q_n(\a_n(\g), \rns_n(\g)) &\le Q_n\left(\erisk(h_{S_n(\bi, *)}; S_n) + 2\nu L, \rns_n(\g)\right) \\
    &\le Q_n\left(\erisk(h_{S_n(\bi, *)}; S_n), \rns_n(\g)\right) + 2\nu L \left(20\sqrt{\frac{\rns_n(\g)}{n}} + 20\sqrt{\frac{\SIsize}{n}} + 15\sqrt{\frac{\ln(\frac{4e^2}{\delta_n})}{n}} + 1 \right) \\
    &\le Q_n\left(\erisk(h_{S_n(\bi, *)}; S_n), \rns_n(\g)\right) + 2\nu L \left(21 + 20\sqrt{\frac{\SIsize}{n}} + 15\sqrt{\frac{\ln(\frac{4e^2}{\delta_n})}{n}}\right) \\
    &:= Q_n\left(\erisk(h_{S_n(\bi, *)}; S_n), \rns_n(\g)\right) + 2\nu L \cdot F_{n}(\SIsize, \delta_n).
\end{align*}
Examining, $F_{n}(\SIsize, \delta_n)$ we note (as shown in the proof of \lemref{Qbound}) that for $n$ sufficiently large (larger than some $N_2(\SIsize, \delta_n)$), we have $F_{n}(\SIsize, \delta_n) \le 22$. Thus,
\beqn
\label{eq:Q_Q}
Q_n(\a_n(\g), \rns_n(\g)) \le Q_n\left(\erisk(h_{S_n(\bi, *)}; S_n), \rns_n(\g)\right) + 44\nu L
.
\eeqn
Combining (\ref{eq:err_R}) and (\ref{eq:Q_Q}),
\beq
&&
\Big\{Q_n(\a_n(\g),\rns_n(\g))
>
R^* + \eps
\;\;\wedge\;\;
\missmass_\g(\bm X_n) \leq \frac{\eps}{18L}
\;\;\wedge\;\;
\rns_n(\g)=d \\
\nonumber
&& \;\;\;
\;\;\wedge\;\;
\erisk(h_{S_n(\bm i, \bm \tbY)}; S_n) \le \erisk(h_{S_n(\bm i, *)}; S_n)
\Big\}
\\
&& \qquad \;\implies\;
\left\{ Q_n(\erisk(h_{S_n(\bm i, *)}; S_n), d) + 44\nu L > \risk(h_{S_n(\bm i, *)}) + \frac{\eps}{2}  \;\;\wedge\;\; |\bm i|=d \right\}.
\eeq
Thus, for all $d \leq \Ng$,
\begin{align}
    \nonumber
    (p_d)_1 &\leq \P\left[
    Q_n( \erisk(h_{S_n(\bm i, *)}; S_n),d )
    >
    \risk(h_{S_n(\bm i, *)}) + \frac{\eps}{2} - 44\nu L
    \;\wedge\;
    |\bm i| = d
    \right]
    \\
    \nonumber
    &\leq
    \P\Big[
     \exists \bm i \in [n]^d
     : Q_n( \erisk(h_{S_n(\bm i, *)}; S_n),d ) \\
    \label{eq:sum-bound-lb}
    &\qquad \qquad \qquad \qquad \qquad > \risk(h_{S_n(\bm i, *)}) +  \frac{\eps}{2} - 44\nu L
    \Big].
\end{align}
To bound the last expression, let $\bm i\in [n]^d$ and denote
\beq
r_{d,n} = \sup_{\alpha \in (0,L)} (Q_n(\alpha,d) - \alpha).
\eeq
We therefore have
\beq
Q_n(\erisk(h_{S_n(\bm i, *)}; S_n) , d) \leq \erisk(h_{S_n(\bm i, *)}; S_n) + r_{d,n}.
\eeq
Let $\bm i' = \{1,\dots,n\}\setminus \bm i$ 
and note that
\beq
\erisk(h_{S_n(\bm i, *)}; S_n) \leq \frac{n-d}{n} \erisk(h_{S_n(\bm i, *)}; S_n(\bm i')) + \frac{d}{n}.
\eeq
Combining the two inequalities above, we get
\beq
Q_n(\erisk(h_{S_n(\bm i, *)}; S_n) , d) \leq \erisk(h_{S_n(\bm i, *)}; S_n(\bm i')) + \frac{d}{n} + r_{d,n}.
\eeq
Recalling $\Ng\in o(n)$, 
by Property \Qthreeb,
\beq
\lim_{n\to\infty} \frac{\Ng}{n} + r_{\Ng,n} =0.
\eeq
In addition, by \bQ2,
we have
$r_{d,n} \leq r_{\Ng,n}$ for all $d\leq \Ng$.
Hence, using our $\nu < \frac{\eps}{176L}$, we take $n$ sufficiently large (larger than some $N_3(\Ng)$), so that for all $d\leq \Ng$,
\beq
\frac{d}{n}  +   r_{d,n}
\leq
\frac{\Ng}{n}  +   r_{\Ng,n}
 \leq \frac{\eps}{4} - 44\nu L,
\eeq 
and thus
\beq
Q_n(\erisk(h_{S_n(\bm i, *)}; S_n) , d) \leq \erisk(h_{S_n(\bm i, *)}; S_n(\bm i')) + \frac{\eps}{4} - 44\nu L .
\eeq
Therefore, for any $n \ge N_0(\nu, \SIsize, \delta_n, \Ng) := \max\{ N_1, N_2, N_3 \}$, 
\beq
&&
Q_n( \erisk(h_{S_n(\bm i, *)}; S_n),d ) 
>
\risk(h_{S_n(\bm i, *)}) +  \frac{\eps}{2} - 44\nu L
\\
&& \qquad\qquad\implies\quad
\erisk(h_{S_n(\bm i, *)}; S_n(\bm i')) 
>
\risk(h_{S_n(\bm i, *)})  +  \frac{\eps}{4}
.
\eeq
Now,
\beqn
\label{eq:exp-bounds}
&& \P\left[
\erisk(h_{S_n(\bm i, *)}; S_n(\bm i')) 
>
\risk(h_{S_n(\bm i, *)})   +  \frac{\eps}{4}
\right]
\\
\nonumber
& = &
\E_{S_n(\bm i)}
\left[
\P_{S_n(\bm i') \gn S_n(\bm i)}
\left[
\erisk(h_{S_n(\bm i, *)}; S_n(\bm i')) 
>
\risk(h_{S_n(\bm i, *)})  +  \frac{\eps}{4}
\right]
\right]
.
\eeqn
Since $\P_{S_n(\bm i') \gn S_n(\bm i)}$ is a product distribution, by Hoeffding's inequality we have that (\ref{eq:exp-bounds}) is bounded above by $e^{-2(n-d)(\frac{\eps}{4})^2}$. 
Since $h_{S_n(\bm i, *)}$ is invariant to permutations of $\bm i$'s entries,
bounding (\ref{eq:sum-bound-lb}) by a union bound over $\bm i$ yields
\beq
(p_d)_1 \leq
\binom{n}{d}
e^{-2(n-d)(\frac{\eps}{4})^2}
\leq
e^{d\log\left(\frac{en}{d}\right)  -2(n-d)(\frac{\eps}{4})^2},
\eeq
where we used
$\binom{n}{d} \leq \left(\frac{en}{d}\right)^d$.
Selecting $n$ large enough so that for all $d\leq \Ng$ we have $d\log(e n/d) \leq (n-d)(\frac{\eps}{4})^2$ and $d \leq n/4$. 
Combining this with \eqref{sum-bound-lb} proves the lemma.
\enpf

\subsection{Proof of \lemref{sublinear_comp}}
\begin{lemma*}
\def\loclabel{\nonumber}
\sublinearcomplem
\end{lemma*}

\bepf
Almost identical to \citet[Lemma 3.7]{hksw21}
---
the former has
a factor of $2$ multiplying 
$|\gnet(\g)|$
---
and hence omitted.
\enpf

\subsection{Proof of \thmref{comp-consist2}}
\label{sec:pf-consist2}
\begin{theorem*}
\compconsisttwothm
\end{theorem*}

\bepf
The claim will follow if we succeed
in showing how
the
lemmas and theorems invoked in proving
\Bcsy\ of \finmednet\
(\thmref{comp-consist})
are applicable in the present setting.

In \lemref{semistablecompressionscheme},
which established that
the procedure is a semi-stable compression scheme
a globally constant $\bits=\log_{2}|\Y|$
was used. The claim remains perfectly true
if the size of the label space happens
to depend on the sample size,
which is precisely how the result is being invoked
in the present setting.

Next, 
we argue that
the $Q_n$ bound in \lemref{Qbound}
remains valid for sufficiently slowly
growing $\bits_n$,
and our chosen rate of $o(n)$ suffices.
The remaining lemmas \ref{lem:richness} and \ref{lem:boundpd} can be used freely since they allow for countable label spaces. Lemmas \ref{lem:sublinear_comp} and \ref{lem:missing_mass} likewise do not require any modifications.
Thus, we have shown how a straightforward
modification of the proof of \thmref{comp-consist}
also proves the theorem in question.
\enpf

\subsection{Proof of \thmref{comp-consist3}}
\label{sec:pf-consist3}
\begin{theorem*}
\compconsistthreethm
\end{theorem*}

\bepf
Let $Q$ be the generalization bound as defined in \lemref{Qbound}, 
and set the input confidence $\delta$ for input size $n$ to $\delta_n$ as stipulated by \Qthreeb.
Choose any $\bits_n \in o(n)$, and $L_n$ such that 
$L_n^2\fns_n, L_n^2\bits_n, L_n^2\ln(\frac{4e^2}{\delta_n}) \in o(n)$.
Similarly, we choose $\fns_n\in o(n)$
arbitrarily (cf. \lemref{Qbound}).
Given a sample $S_n\sim \bmu^n$, we abbreviate the optimal empirical error $\trunc{\a_n^*}=\a(\g^*_n)$ and the optimal compression size $\trunc{\rns_n^*}=\rns(\g^*_n)$ as computed by \algref{simple} on our truncated sample. 
Let
$\trunc{h_n}$
be the output of \ctblunbdd,
\beq
\trunc{f_n^*}
:=
\argmin_{f:\X\to\Y\wedge L_n}R(f),
\eeq
be the Bayes-optimal predictor
for the truncated label space and
\beq
\trunc{R_n^*}:=
R(\trunc{f_n^*}).
\eeq
be the optimal risk in the \emph{truncated setting} on our modified sample.
For brevity we denote
\beq
Q_n(\alpha,\fns) := Q_n(\alpha,\fns,\SIsize, \delta_n, L_n).
\eeq
To prove \thmref{comp-consist3}, we first follow the standard technique, used also in \thmref{comp-consist}, of decomposing the excess error over the Bayes error into two terms:
\beq
\risk(\trunc{h_n}) - \trunc{R(f_n^*)}
&= &
\big(\risk(\trunc{h_n}) - Q_n(\a_n^*,\rns_n^*) \big)
+
\big(Q_n(\a_n^*,\rns_n^*) - \trunc{R(f_n^*)}\big)
\\
&=:&
T\subI(n) + T\subII(n).
\eeq
We now show that each term decays to zero almost surely.
Regarding, $T\subI(n)$
we wish to invoke
Lemmas
\ref{lem:semistablecompressionscheme} and \ref{lem:Qbound}
as in the proof of \thmref{comp-consist}.
The argument of \lemref{semistablecompressionscheme},
which shows that \countable\ is a semi-stable compression scheme,
applies verbatim to
\ctblunbdd.

As for
\lemref{Qbound},
properties \bQ1\ and \bQ2\ trivially
continue to hold,
while to ensure 
\Qthreeb,
we constrain the choice of the
diameter truncation schedule $L_n$
to satisfy
$L_n^2\fns_n, L_n^2\bits_n, L_n^2\ln(\frac{4e^2}{\delta_n}) \in o(n)$,
where $\fns_n\in o(n)$ is as in \lemref{Qbound}. %
Having verified the applicability of
Lemmas \ref{lem:semistablecompressionscheme} and \ref{lem:Qbound}, we invoke property \bQ1\ from \lemref{Qbound}:
\beq
\P\!
\left[
\risk(\trunc{h_n}) - Q_n(\a_n^*,\rns_n^*) > 0
\right] 
\leq \delta_n,
\qquad n\ge1.
\eeq
Since
$\delta_n$ was chosen as
furnished by
\lemref{Qbound}(\Qthreeb),
the Borel-Cantelli lemma 
implies that\\
$\limsup_{n\to\infty} T\subI(n) \leq 0$ with probability $1$.

We now proceed to argue that
the generalization bound $Q_n(\a_n^*, \rns_n^*)$ approaches the truncated optimal Bayes error %
$\trunc{R_n^*}$, 
which will establish
$\limsup_{n\to\infty} T\subII(n) \leq 0$ almost surely.
Since Lemmas \ref{lem:richness}, \ref{lem:boundpd}, \ref{lem:sublinear_comp} and \ref{lem:missing_mass} do not rely on $L_n$ being fixed, the are applicable to our setting. Thus,
the argument from the proof of
\thmref{comp-consist}
applies here as well,
and thus
$\limsup_{n\to\infty} T\subII(n) \leq 0$ almost surely.
It follows that
$\lim_{n\to\infty} \risk(\trunc{h_n}) - \trunc{R_n^*} = 0$.
It remains to exploit the BIE property of $\Y$
and invoke \thmref{trunc-risk}
to conclude that
$\lim_{n\to\infty} \trunc{R(f_n^*)} - R^* = 0$, whence
\beq \lim_{n\to\infty} \risk(\trunc{h_n}) - R^* = 0 .
\eeq
\enpf

\section{Compression Scheme Theorems}
\label{ap:compression_scheme_theorems}

In this section we introduce a series of new theorems regarding semi-stable compression schemes, each leading to the next, culminating in \thmref{agnosemitstablewithinterpol} which is used to bound the generalization of \medoidnet.

\subsection{Definitions}

\subsection{Setting}
Let $\X$ be an instance space, and 
$\Y$ of finite diameter, and $\bmu$ a distribution 
supported on the product Borel $\sigma$-algebra of $\X\times\Y$, such that $\Y$ is bounded by some $L>0$: $\forall y_1, y_2 \in \Y: \ell(y_1,y_2) \le L$.

\subsection{Theorems \& Lemmas}

\subsubsection{Semi-stable Agnostic sample compression scheme of given size}

\begin{theorem} \label{thm:agnosemitstablegivensize}
For any $k, b\in \mathbb{N} \cup \{0\}$, let $(\compfunc, \reconfunc)$ be any semi-stable compression scheme of size at most $k$ using at most $b$ bits of side-information. For any distribution $\bmu$ over $\X \times \Y$, any $n \in \mathbb{N}$ with $n > 2k$, and any $\delta \in (0,1)$, for $S_n \sim \bmu^n$, with probability at least $1 - \delta$

\begin{align*}
    \left| \risk(\reconfunc(\compfunc(S_n))) - \erisk(\reconfunc(\compfunc(S_n)); S_n) \right| &\le \sqrt{\frac{4L^{2}}{n-2k} \left(k\ln (4) + \ln (4/\delta)\right)} + \sqrt{\frac{L^{2}}{n-2k} b \ln (2)}.
\end{align*}

\end{theorem}

\bepf
If $k=0$, the result trivially follows from Hoeffding's inequality, so let us suppose $k \ge 1$.  As in the proof of Theorem 5 in \citet{BousquetHMZ20}, fix any $T_n \in [n-1]$ and let $\I_n$ be any family of subsets of $[n]$ with the properties that each $I\in\I_n$ has $|I| \le n-T_n$, and for every $i_1, \dots, i_k \in [n]$ there exists $I\in\I_n$ such that $\{i_1,\dots,i_k\} \subseteq I$. In particular, \citeauthor{BousquetHMZ20} construct a family $\I_n$ satisfying the properties above with $ T_n = k\lfloor \frac{n}{2k} \rfloor$, and with $ |\I_n| = \binom{2k}{k} < 4^k$: namely, let $D_1,\dots,D_{2k}$ be any partition of $[n]$ with each $|D_i| \in \{\lfloor \frac{n}{2k} \rfloor,\lceil \frac{n}{2k} \rceil)\}$, and define $\I_n = \left\{\bigcup \{D_j : j \in \mathcal{J} \} : \mathcal{J} \subseteq [2k], |\mathcal{J} | = k\right\}$; that is, $\I_n$ contains all unions of exactly $k$ of the $2k$ sets $D_j$.

Let there be a sample $S_n \sim \bmu^{n} $. Recall our notation that for any $I \in \I_n$ we have $S(I) \subseteq S_n$.
For any $I \in \I_n$, define $\bar{I} := [n] \setminus I$. Let $\sigma : [n] \to [n]$ be a uniformly random permutation of $[n]$, and for any $I=(i_1, \dots, i_\ell) \subseteq [n]$ define $\bsigma(I) := (\sigma(i_1), \dots, \sigma(i_\ell))$.

Next, for any $I \subseteq [n]$ and any $\bb \in \{0,1\}^{b}$, let $\hat{h}_{I,\bb} := \reconfunc\left( \compfunccs\left(S\left(I\right)\right), \bb \right)$. Now, since $S(\bar{I})$ is independent of $S(I)$, Hoeffding's Inequality (applied under the conditional distribution given $S(I)$) and the law of total probability imply that with probability $1-\frac{\delta}{2|\I_n| \cdot 2^{b}}$:

\begin{align*}
    \left| \risk\left(\hat{h}_{I,\bb} \right) - \erisk\left(\hat{h}_{I,\bb}; S\left(\bar{I}\right)\right) \right| &\le \sqrt{\frac{L^{2}\ln (\frac{4|\I_n|\cdot 2^b}{\delta})}{2(n-|I|)}} \\
    &= \sqrt{\frac{L^{2}\left(\ln (|\I_n|) + b \ln (2) + \ln (4/\delta)\right)}{2(n-|I|)}}
    .
\end{align*}

Applying this under the conditional distribution given $\sigma$, together with the union bound and the law of total probability, we have that with probability at least $1-\frac{\delta}{2}$, every $I\in\I_n$ and every $\bb \in \{0,1\}^{b}$ has
\begin{align*}
    \left| \risk\left(\hat{h}_{\bsigma^{-1}\left(I\right),\bb} \right) - \erisk\left(\hat{h}_{\bsigma^{-1}\left(I\right),\bb}; S\left(\overline{\bsigma^{-1}\left(I\right)}\right)\right) \right| &\le \sqrt{\frac{L^{2}\left(\ln (|\I_n|) + b \ln (2) + \ln (4/\delta)\right)}{2(n-|I|)}}
    .
\end{align*}

In particular, let $\bi^*$ be the indices such that 
\beq
\compfunccs(S_n) = \left\{ (X_i, Y_i) \in S_n: i \in \bi^* \right\} = S_n(\bi^*)
.
\eeq
Now, by property (ii) of $\I_n$ there must exist $I^*\in \I_n$ such that
\beq \bsigma(\bi^*) \subseteq I^*, \eeq
which means:
\begin{align*}
    \Rightarrow \quad & \bi^* \subseteq \bsigma^{-1}(I^*) \\
    \Rightarrow \quad & \compfunccs(S_n) = S_n(\bi^*) \subseteq S_n(\bsigma^{-1}(I^*))
    .
\end{align*}
Due to the semi-stability property of $(\compfunc, \reconfunc)$ and since $S_n(\bsigma^{-1}(I^*)) \subseteq S_n$, this implies

\beq
    \Rightarrow \reconfunc(\compfunccs(S_n(\bsigma^{-1}(I^*))), \compfuncsi(S_n)) = \reconfunc(\compfunccs(S_n), \compfuncsi(S_n)) = \reconfunc(\compfunc\left( S_n \right)).
\eeq
Thus, on the above event of probability at least $1-\frac{\delta}{2}$ we get, 
for $I=I^*$ and $\bb=\compfuncsi(S_n)$,
\beq
\left| \risk(\reconfunc(\compfunc\left( S_n \right))) - \erisk(\reconfunc(\compfunc\left( S_n \right)); S(\overline{\bsigma^{-1}(I^*)})) \right| \le \sqrt{\frac{L^{2}\left(\ln (|\I_n|) + b \ln (2) + \ln (4/\delta)\right)}{2(n-|I^*|)}}.
\eeq

Furthermore, by property (i) of $I_n$ we have that $n - |I^*| \ge T_n$, and so

\beqn
\label{eq:serr_11}
\left| \risk(\reconfunc(\compfunc\left( S_n \right))) - \erisk(\reconfunc(\compfunc\left( S_n \right)); S(\overline{\bsigma^{-1}(I^*)})) \right| \le \sqrt{\frac{L^{2}\left(\ln (|\I_n|) + b \ln (2) + \ln (4/\delta)\right)}{2T_n}}.
\eeqn

Next, we want to relate $\erisk(\reconfunc(\compfunc\left( S_n \right)); S(\overline{\bsigma^{-1}(I^*)}))$ to $\erisk(\reconfunc(\compfunc\left( S_n \right)); S_n)$. Let $\hat{h} := \reconfunc(\compfunc\left( S_n \right))$. For each $i \in [n]$, let $\ell_i := \ell(\hat{h}(X_i), Y_i)$. For any $I \in \I_n$, by Hoeffding's inequality without replacement \citep{Bardenet15} applied under the conditional distribution given $S_n$, together with the law of total probability, with probability at least $1 - \frac{\delta}{2|\I_n|}$ it holds that

\beq
    \left|\frac{1}{n - |I|} \sum_{i\in \overline{\bsigma^{-1}(I)}} \ell_i - \erisk(\reconfunc(\compfunc\left( S_n \right)); S_n) \right| \le \sqrt{\frac{L^2 \ln (4|\I_n|/\delta)}{2(n-|I|)}}.
\eeq

By the union bound, this holds simultaneously for all $I\in \I_n$ with probability at least $1-\frac{\delta}{2}$. In particular, taking $I=I^*$, and recalling that $n-|I^*| \ge T_n$, on this event we have that
\beqn
\label{eq:serr_21}
\left|\erisk(\reconfunc(\compfunc\left( S_n \right)); S(\overline{\bsigma^{-1}(I^*)})) - \erisk(\reconfunc(\compfunc\left( S_n \right)); S_n) \right| \le \sqrt{\frac{L^{2}\ln (4|\I_n|/\delta)}{2T_n}}
\eeqn
By the union bound, the above two events (each of probability at least $1 - \frac{\delta}{2}$) hold simultaneously with probability at least $1 - \delta$, in which case \eqref{serr_11} and \eqref{serr_21} together imply
\begin{align*}
    \left| \risk(\reconfunc(\compfunc\left( S_n \right))) - \erisk(\reconfunc(\compfunc\left( S_n \right)); S_n) \right| &\le \left| \risk(\reconfunc(\compfunc\left( S_n \right))) - \erisk(\reconfunc(\compfunc\left( S_n \right)); S(\overline{\bsigma^{-1}(I^*)})) \right| + \\
    & \quad + \left|\erisk(\reconfunc(\compfunc\left( S_n \right)); S(\overline{\bsigma^{-1}(I^*)})) - \erisk(\reconfunc(\compfunc\left( S_n \right)); S_n) \right| \\
    &\le \sqrt{\frac{L^{2}\left(\ln (|\I_n|) + b \ln (2) + \ln (4/\delta)\right)}{2T_n}} + \sqrt{\frac{L^{2}\ln (4|\I_n|/\delta)}{2T_n}} \\
    &\le \sqrt{\frac{4L^{2}\left(\ln (|\I_n|) + \ln (4/\delta)\right)}{2T_n}} + \sqrt{\frac{L^{2} b \ln (2)}{2T_n}}
\end{align*}
The theorem now immediately follows from plugging the aforementioned family $\I_n$ from \citet{BousquetHMZ20}, having $ |\I_n|= \binom{2k}{k} < 4^k$ and $ T_n = k \floor{\frac{n}{2k}} > \frac{n-2k}{2}$.
\enpf

\subsubsection{Accounting for realizable case}

\begin{lemma} \label{lem:empiricalbernstein}
Let $Z_1, \dots, Z_n$ be i.i.d random variables with values in $[0,L]$ for some $L>0$, and let $\delta > 0$. Define $\bZ := \frac{1}{n}\sum_{i=1}^{n} Z_i$. Then with probability at least $1 - \delta$ we have:

\beq
\E[\bZ] - \bZ \le \bZ \cdot \sqrt{\frac{2\ln(4/\delta)}{n-1}} + L \sqrt{\frac{2\ln(4/\delta)}{n-1}} + \frac{7L\ln(4/\delta)}{3(n-1)}.
\eeq
\end{lemma}

\bepf
Let $\tZ_i:=\frac{Z_i}{L} \in [0,1]$ for all $i\in \{1,\dots, n\}$. Using the empirical Bernstein inequality stated in \citet[Theorem 4]{maurer2009empirical}, we get
\begin{align}
    \nonumber
    \E[\btZ] - \btZ &\le \sqrt{\frac{2 \sVar(\boldsymbol{\tZ}) \ln(4/\delta)}{n}} + \frac{7\ln(4/\delta)}{3(n-1)} \\
    \label{eq:empbernstein}
    \Rightarrow \E[\bZ] - \bZ &\le L\sqrt{\frac{2\sVar(\boldsymbol{Z}) \ln(4/\delta)}{L^2n}} + \frac{7L\ln(4/\delta)}{3(n-1)} \\
    \nonumber
    &= \sqrt{\frac{2 \sVar(\boldsymbol{Z}) \ln(4/\delta)}{n}} + \frac{7L\ln(4/\delta)}{3(n-1)},
\end{align}
where $\sVar(\boldsymbol{Z})$ is defined to be:
\beq
\sVar(\boldsymbol{Z}) := \frac{1}{n(n-1)} \sum_{1\le i < j \le n} (Z_i-Z_j)^{2}.
\eeq
Observe that
\begin{align*}
    \sVar(\boldsymbol{Z}) &= \frac{1}{n(n-1)} \sum_{1\le i < j \le n} (Z_i-Z_j)^{2} \\
    &= \frac{1}{n(n-1)} \cdot n \sum_{i=1}^{n} (Z_i-\bZ)^{2} \\
    &= \frac{1}{n-1} \sum_{i=1}^{n} (Z_i-\bZ)^{2} \\
    &= \frac{1}{n-1} \left(\sum_{i=1}^{n} Z_{i}^{2} -2\bZ \sum_{i=1}^{n} Z_i + n\bZ^{2} \right)\\
    &= \frac{1}{n-1} \left(\sum_{i=1}^{n} Z_{i}^{2} -2n\bZ^{2} + n\bZ^{2} \right)\\
    &= \frac{1}{n-1} \left(\sum_{i=1}^{n} Z_{i}^{2} - n\bZ^{2} \right)\\
    &\le \frac{1}{n-1} \left(nL^{2} - n\bZ^{2} \right)\\
    &= \frac{n}{n-1} \left(L^{2} - \bZ^{2} \right).
\end{align*}
Now plugging this back into \eqref{empbernstein}:
\begin{align*}
    \E[\bZ] - \bZ &\le \sqrt{\frac{2(L^2-\bZ^2) \ln(4/\delta)}{n-1}} + \frac{7L\ln(4/\delta)}{3(n-1)} \\
    &\le \sqrt{\frac{2 (L^2+\bZ^2) \ln(4/\delta)}{n-1}} + \frac{7L\ln(4/\delta)}{3(n-1)} \\
    &\le \bZ  \sqrt{\frac{2\ln(4/\delta)}{n-1}} + L \sqrt{\frac{2\ln(4/\delta)}{n-1}} + \frac{7L\ln(4/\delta)}{3(n-1)},
\end{align*}
where the last inequality used $\sqrt{x+y} \le \sqrt{x} + \sqrt{y}$ for any $x,y>0$.
\enpf

\begin{lemma} \label{lem:samplingwithoutreplacement}
Let $S_n = \{(X_i, Y_i)\}_{i=1}^{n} \sim (\X \times \Y)^n$ be a \emph{given} sample. Let $\hat{h}: \X \to \Y$ be a predictor, and $\ell: \Y \times \Y \to \mathbb{R}^+$ be a bounded loss function by some $L>0$. Let $I \subseteq [n]$ be a random variable \emph{sampled without replacement} from $[n]$. Then for any $\delta \in (0,1)$, with confidence at least $1-\delta$:

\beq
    \left| \frac{1}{|I|} \sum_{i\in I} \ell\left(\hat{h}(X_i),Y_i\right) - \erisk(\hat{h}; S_n) \right| \le \erisk(\hat{h}; S_n) \sqrt{\frac{2\ln(2/\delta)}{|I|}} + L\sqrt{\frac{2\ln(2/\delta)}{|I|}} + \frac{2L\ln(2/\delta)}{3|I|}.
\eeq
\end{lemma}

\bepf
Let $\ell_i := \ell(\hat{h}(X_i),Y_i)$
for $i \in [n]$.
Treating $\ell_1, \dots, \ell_n$ as a given finite population, and $\{\ell_i\}_{i\in I}$ as a random sample drawn without replacement from it, we can use
a version of Bernstein's inequality \citep{Bardenet15}, 
\beq
    \P\left( \left| \frac{1}{|I|} \sum_{i \in I} \ell_i - \mu \right| \ge \varepsilon \right) \le 2\exp\left(- \frac{|I| \varepsilon^2}{2\sigma^{2} + \frac{2L}{3} \varepsilon} \right)
    \qquad
\varepsilon > 0,
\eeq
where we have defined:
\begin{align*}
    \mu &:= \frac{1}{n}\sum_{i=1}^{n} \ell_i = \erisk(\hat{h}; S_n) & \text{(population mean)}\\
    \sigma^{2} &:= \frac{1}{n} \sum_{i=1}^{n} (\ell_i - \mu)^{2} & \text{(population variance)}.
\end{align*}
For any $\delta \in (0,1)$, we get that with confidence at least $1-\delta$:
\beq
    \left| \frac{1}{|I|} \sum_{i \in I} \ell_i - \mu \right| \le \frac{2L}{3|I|} \ln(2/\delta) + \sqrt{\frac{2\sigma^{2}\ln(2/\delta)}{|I|}}.
\eeq
Now similarly to the analysis in the proof of \lemref{empiricalbernstein}, we see that
\beq \sigma^{2} \le L^{2} + \erisk^{2}(\hat{h}; S_n)\eeq
using that and that $\sqrt{x+y} \le \sqrt{x} + \sqrt{y}$ for any $x,y>0$ we get the statement of the lemma.
\enpf

\begin{theorem} \label{thm:agnosemitstablegivensizewithinterpol}
For any $k \in \mathbb{N}$,  $b\in \mathbb{N} \cup \{0\}$, let $(\compfunc, \reconfunc)$ be any semi-stable compression scheme of size at most $k$ using at most $b$ bits of side-information. For any distribution $\bmu$ over $\X \times \Y$, any $n \in \mathbb{N}$ with $n > 4k + 4$, and any $\delta \in (0,1)$, for $S_n \sim \bmu^n$, with probability at least $1 - \delta$

\begin{align*}
    \risk(\reconfunc(\compfunc(S_n))) - \erisk(\reconfunc(\compfunc(S_n)); S_n) &\le \erisk(\reconfunc(\compfunc(S_n)); S_n) \left( 5 \sqrt{\frac{8\left(\ln(\frac{4}{\delta}) + k\ln 4 \right)}{n}} + 4\sqrt{\frac{8b\ln 2}{n}} \right) \\
    &\quad + 2L\sqrt{\frac{8\left(\ln(\frac{4}{\delta}) + k\ln 4 \right)}{n}} + \frac{(28+8L)\left(\ln(\frac{4}{\delta}) + k\ln 4 \right)}{3n}\\
    &\quad + L\sqrt{\frac{8b\ln 2}{n}} + \frac{28b\ln 2}{3n}.
\end{align*}

\end{theorem}

\bepf
Similarly to \citet{DBLP:conf/alt/HannekeK21}, this proof follows similar arguments as \thmref{agnosemitstablegivensize}, except using \lemref{empiricalbernstein} in place of Hoeffding's inequality in both places in the proof where such inequalities are used.

Let $\I_n$ and $T_n$ be as in the proof of \thmref{agnosemitstablegivensize}, and let $[n] = \{1, \dots, n\}$ for and $n \in \mathbb{N}$.

Let there be a sample $S_n = \{(X_i, Y_i)\}_{i=1}^{n} \sim \bmu^{n} $. As in \thmref{agnosemitstablegivensize}, for any $I \in \I_n$ we have $S(I) = \{(X_i, Y_i) : i \in I\} \subseteq S_n$, for any $I \in \I_n$, define $\bar{I} := [n] \setminus I$. Let $\sigma : [n] \to [n]$ be a uniform random permutation of $[n]$, and for any $I=(i_1, \dots, i_\ell) \subseteq [n]$ define $\bsigma(I) := (\sigma(i_1), \dots, \sigma(i_\ell))$.

For any $I \subseteq [n]$ and any $\bb \in \{0,1\}^{b}$, let $\hat{h}_{I,\bb} := \reconfunc\left( \compfunccs\left(S\left(I\right)\right), \bb \right)$. Now, since $S(\bar{I})$ is independent of $S(I)$, \lemref{empiricalbernstein} (applied under the conditional distribution given $S(I)$) and the law of total probability imply that with probability $1-\frac{\delta}{2|\I_n| \cdot 2^{b}}$:
\begin{align*}
    \risk\left(\hat{h}_{I,\bb} \right) - \erisk\left(\hat{h}_{I,\bb}; S\left(\bar{I}\right)\right) &\le \erisk\left(\hat{h}_{I,\bb}; S\left(\bar{I}\right)\right) \sqrt{\frac{2\ln(\frac{4|\I_n|\cdot 2^b}{\delta})}{n-|I|-1}}\\
    &\quad + L \sqrt{\frac{2\ln(\frac{4|\I_n|\cdot 2^b}{\delta})}{n-|I|-1}} + \frac{7\ln(\frac{4|\I_n|\cdot 2^b}{\delta})}{3(n-|I|-1)}.
\end{align*}
Applying this under the conditional distribution given $\sigma$, together with the union bound and the law of total probability, we have that with probability at least $1-\frac{\delta}{2}$, every $I\in\I_n$ and every $\bb \in \{0,1\}^{b}$ has
\begin{align*}
    \risk\left(\hat{h}_{\bsigma^{-1}\left(I\right),\bb} \right) - \erisk\left(\hat{h}_{\bsigma^{-1}\left(I\right),\bb}; S\left(\overline{\bsigma^{-1}\left(I\right)}\right)\right) &\le \erisk\left(\hat{h}_{\bsigma^{-1}\left(I\right),\bb}; S\left(\overline{\bsigma^{-1}\left(I\right)}\right)\right) \sqrt{\frac{2\ln(\frac{4|\I_n|\cdot 2^b}{\delta})}{n-|I|-1}} \\
    &\quad + L \sqrt{\frac{2\ln(\frac{4|\I_n|\cdot 2^b}{\delta})}{n-|I|-1}} + \frac{7\ln(\frac{4|\I_n|\cdot 2^b}{\delta})}{3(n-|I|-1)}.
\end{align*}
In particular, let $\bi^*$ be the indices such that 
\beq\compfunccs(S_n) = \left\{ (X_i, Y_i) \in S_n: i \in \bi^* \right\} = S_n(\bi^*)\eeq
Now, by property (ii) of $\I_n$ there must exist $I^*\in \I_n$ such that
\beq \bsigma(\bi^*) \subseteq I^*, \eeq
which means:
\begin{align*}
    \Rightarrow \quad & \bi^* \subseteq \bsigma^{-1}(I^*) \\
    \Rightarrow \quad & \compfunccs(S_n) = S_n(\bi^*) \subseteq S_n(\bsigma^{-1}(I^*))
    .
\end{align*}
Due to the semi-stability property of $(\compfunc, \reconfunc)$ and since $S_n(\bsigma^{-1}(I^*)) \subseteq S_n$, this implies
\beq
    \Rightarrow \reconfunc(\compfunccs(S_n(\bsigma^{-1}(I^*))), \compfuncsi(S_n)) = \reconfunc(\compfunccs(S_n), \compfuncsi(S_n)) = \reconfunc(\compfunc\left( S_n \right))
    .
\eeq
Thus, on the above event of probability at least $1-\frac{\delta}{2}$ we can get for $I=I^*$ and $\bb=\compfuncsi(S_n)$,
\begin{align*}
    \risk\left(\reconfunc(\compfunc\left( S_n \right)) \right) - \erisk\left(\reconfunc(\compfunc\left( S_n \right)); S\left(\overline{\bsigma^{-1}\left(I^*\right)}\right)\right) &\le \erisk\left(\reconfunc(\compfunc\left( S_n \right)); S\left(\overline{\bsigma^{-1}\left(I^*\right)}\right)\right) \sqrt{\frac{2\ln(\frac{4|\I_n|\cdot 2^b}{\delta})}{n-|I^*|-1}} \\
    &\quad + L \sqrt{\frac{2\ln(\frac{4|\I_n|\cdot 2^b}{\delta})}{n-|I^*|-1}} + \frac{7\ln(\frac{4|\I_n|\cdot 2^b}{\delta})}{3(n-|I^*|-1)}.
\end{align*}
Furthermore, by property (i) of $I_n$ we have that $n - |I^*| \ge T_n$, and that $\erisk\left(\reconfunc(\compfunc\left( S_n \right)); S\left(\overline{\bsigma^{-1}\left(I^*\right)}\right)\right) \le \frac{n}{T_n} \erisk\left(\reconfunc(\compfunc\left( S_n \right)); S_n\right)$, so
\begin{align}
\label{eq:serr_12}
    \risk\left(\reconfunc(\compfunc\left( S_n \right)) \right) - \erisk\left(\reconfunc(\compfunc\left( S_n \right)); S\left(\overline{\bsigma^{-1}\left(I^*\right)}\right)\right) &\le \frac{n}{T_n}\erisk\left(\reconfunc(\compfunc\left( S_n \right)); S_n\right) \sqrt{\frac{2\ln(\frac{4|\I_n|\cdot 2^b}{\delta})}{T_n-1}} \\
    \nonumber
    &\quad + L \sqrt{\frac{2\ln(\frac{4|\I_n|\cdot 2^b}{\delta})}{T_n-1}} + \frac{7\ln(\frac{4|\I_n|\cdot 2^b}{\delta})}{3(T_n-1)}.
\end{align}
Now, given $S_n$, we define $\ell_i := \ell(\reconfunc(\compfunc(S_n))(X_i), Y_i)$
for $i \in [n]$. Now for any $I \in \I_n$, under the conditional distribution given $S_n$ we apply \lemref{samplingwithoutreplacement} and see that with probability at least $1-\frac{\delta}{2|\I_n|}$:
\begin{align*}
    \left| \frac{1}{n-|I|} \sum_{i \in \overline{\bsigma^{-1}\left(I\right)}} \ell_i - \erisk(\reconfunc(\compfunc(S_n)); S_n) \right| &\le \erisk(\reconfunc(\compfunc(S_n)); S_n) \sqrt{\frac{2\ln(4|\I_n|/\delta)}{n-|I|}} \\
    &\quad + L\sqrt{\frac{2\ln(4|\I_n|/\delta)}{n-|I|}} + \frac{2L\ln(4|\I_n|/\delta)}{3(n-|I|)}.
\end{align*}
By the union bound, this holds simultaneously for all $I \in \I_n$ with probability at least $1- \frac{\delta}{2}$. In particular, taking $I = I^*$, and recalling that $n - |I^*| \ge T_n$, on this event we have that
\begin{align}
\label{eq:serr_22}
    \left| \erisk(\reconfunc(\compfunc(S_n)); S(\overline{\bsigma^{-1}\left(I^*\right)})) - \erisk(\reconfunc(\compfunc(S_n)); S_n) \right| &\le \erisk(\reconfunc(\compfunc(S_n)); S_n) \sqrt{\frac{2\ln(4|\I_n|/\delta)}{T_n}} \\
    \nonumber
    &\quad + L\sqrt{\frac{2\ln(4|\I_n|/\delta)}{T_n}} + \frac{2L\ln(4|\I_n|/\delta)}{3T_n}
    .
\end{align}
By the union bound, the two events represented by \eqref{serr_12} and \eqref{serr_22} hold simultaneously with probability at least $1- \delta$, in which case together we get:
\begin{align*}
    \risk(\reconfunc(\compfunc(S_n))) - \erisk(\reconfunc(\compfunc(S_n)); S_n)  &\le \risk\left(\reconfunc(\compfunc\left( S_n \right)) \right) - \erisk\left(\reconfunc(\compfunc\left( S_n \right)); S\left(\overline{\bsigma^{-1}\left(I^*\right)}\right)\right) \\
    &\quad + \left| \erisk(\reconfunc(\compfunc(S_n)); S(\overline{\bsigma^{-1}\left(I^*\right)})) - \erisk(\reconfunc(\compfunc(S_n)); S_n) \right| \\
    &\le \erisk(\reconfunc(\compfunc(S_n)); S_n) \left( \left( 1+\frac{n}{T_n} \right) \sqrt{\frac{2\ln(\frac{4|\I_n|}{\delta})}{T_{n}-1}} + \frac{n}{T_n}\sqrt{\frac{2b\ln 2}{T_n-1}} \right) \\
    &\quad + 2L\sqrt{\frac{2\ln(\frac{4|\I_n|}{\delta})}{T_n-1}} + \frac{(7+2L)\ln(\frac{4|\I_n|}{\delta})}{3(T_n-1)} + L\sqrt{\frac{2b\ln 2}{T_n-1}} + \frac{7b\ln 2}{3(T_n-1)}.
\end{align*}
The theorem now follows from plugging the aforementioned family $\I_n$ from \citet{BousquetHMZ20}, with $ |\I_n|= \binom{2k}{k} < 4^k$ and $ T_n = k \floor{\frac{n}{2k}} > \frac{n-2k}{2}$
---
with the modification thatsince $n > 4k + 4$ we have $\frac{n-2k}{2}-1> \frac{n}{4}$, meaning $T_n > T_n -1 > \frac{n}{4}$:
\begin{align*}
    \risk(\reconfunc(\compfunc(S_n))) - \erisk(\reconfunc(\compfunc(S_n)); S_n) &\le \erisk(\reconfunc(\compfunc(S_n)); S_n) \left( 5 \sqrt{\frac{8\left(\ln(\frac{4}{\delta}) + k\ln 4 \right)}{n}} + 4\sqrt{\frac{8b\ln 2}{n}} \right) \\
    &\quad + 2L\sqrt{\frac{8\left(\ln(\frac{4}{\delta}) + k\ln 4 \right)}{n}} + \frac{(28+8L)\left(\ln(\frac{4}{\delta}) + k\ln 4 \right)}{3n}\\
    &\quad + L\sqrt{\frac{8b\ln 2}{n}} + \frac{28b\ln 2}{3n}.
\end{align*}

\enpf

\subsubsection{Semi-stable compression scheme of bounded sample-dependent size}

\begin{theorem} \label{thm:agnosemitstableboundedsizewithinterpol}
Let $(\compfunc, \reconfunc)$ be any semi-stable compression scheme with side-information. For any distribution $\bmu$ over $\X \times \Y$, any $n \in \mathbb{N}$, and any $\delta \in (0,1)$, for $S_n \sim \bmu^n$, with probability at least $1 - \delta$, if $|\compfunccs(S_n)| < \frac{n}{4}-1$ then

\begin{align*}
    \risk(\reconfunc(\compfunc(S_n))) - \erisk(\reconfunc(\compfunc(S_n)); S_n) &\le
    \erisk(\reconfunc(\compfunc(S_n)); S_n) \left( 5 \sqrt{\frac{8T_{\compfunc, S_n}}{n}} + 4\sqrt{\frac{8|\compfuncsi(S_n)|\ln 2}{n}} \right) \\
    &\quad + 2L\sqrt{\frac{8T_{\compfunc, S_n}}{n}} + \frac{(28+8L)T_{\compfunc, S_n}}{3n} \\
    &\quad + L\sqrt{\frac{8|\compfuncsi(S_n)|\ln 2}{n}} + \frac{28|\compfuncsi(S_n)|\ln 2}{3n} ,
\end{align*}

where 

\beq
    T_{\compfunc, S_n} := \ln\left(\frac{4(|\compfunccs(S_n)|+1)(|\compfunccs(S_n)|+2)(|\compfuncsi(S_n)|+1)(|\compfuncsi(S_n)|+2)}{\delta}\right) + |\compfunccs(S_n)|\ln 4.
\eeq

\end{theorem}

\bepf

Let $(\compfunc, \reconfunc)$ be any semi-stable compression scheme with side-information. Now for each $k \in \mathbb{N} \cup \{0\}$ and $b \in \mathbb{N} \cup \{0\}$, let $(\compfunc_{k,b}, \reconfunc)$ be a compression scheme such that, for any 
$S_n \sim \mu^n$, if $|\compfunccs(S)|\le k$ and $|\compfuncsi(S)|\le b$, then $\compfunc_{k,b}(S) = \compfunc(S)$, and otherwise $(\compfunccs)_{k,b}(S_n) = \emptyset$ and $(\compfuncsi)_{k,b}(S_n) = \emptyset$. In particular, note that $|(\compfunccs)_{k,b}(S)|\le k$ and $|(\compfuncsi)_{k,b}(S)|\le b$ always. Thus, for each $k$ and $b$, \thmref{agnosemitstablegivensizewithinterpol} implies that for any 
$n
> 4k+4$ 
\begin{align*}
    \left| \risk(\reconfunc(\compfunc(S_n))) - \erisk(\reconfunc(\compfunc(S_n)); S_n) \right| &\le \erisk(\reconfunc(\compfunc(S_n)); S_n) \left( 5 \sqrt{\frac{8t_{k,b}}{n}} + 4\sqrt{\frac{8b\ln 2}{n}} \right) \\
    &\quad + 2L\sqrt{\frac{8t_{k,b}}{n}} + \frac{(28+8L)t_{k,b}}{3n}\\
    &\quad + L\sqrt{\frac{8b\ln 2}{n}} + \frac{28b\ln 2}{3n} 
\end{align*}
holds
with probability at least $1 - \frac{\delta}{(k+1)(k+2)(b+1)(b+2)}$,
where 
\beq
    t_{k,b} := \ln\left(\frac{4(k+1)(k+2)(b+1)(b+2)}{\delta}\right) + k\ln 4.
\eeq

By the union bound, the above claim holds simultaneously for all $k \in \mathbb{N} \cup \{0\}$ and $b \in \mathbb{N} \cup \{0\}$ with probability at least $1 - \sum_{k,b}\frac{\delta}{(k+1)(k+2)(b+1)(b+2)} = 1 - \delta$. Finally, note that there necessarily exists some $k \in \mathbb{N} \cup \{0\}$ and $b \in \mathbb{N} \cup \{0\}$ for which $|\compfunccs(S)| = k$ and $|\compfuncsi(S)| = b$, in which case $\reconfunc(\compfunc(S)) = \reconfunc(\compfunc_{k,b}(S))$ for these $k$ and $b$. 
This completes the proof.
\enpf

\subsubsection{Semi-stable compression scheme of any sample-dependent size
}
\label{ap:agnosemitstablewithinterpol}

\begin{theorem*}\textup{\ref{thm:agnosemitstablewithinterpol}}.
\agnosemitstablewithinterpol
\end{theorem*}

\bepf
Let $(\compfunc, \reconfunc)$ be any semi-stable compression scheme with side-information. We will observe the RHS of \thmref{agnosemitstableboundedsizewithinterpol}. Since for any $x \ge \sqrt{3}$ $\ln(x^2) < \frac{x}{2}$, firstly we have:
\begin{align*}
    T_{\compfunc, S_n} &:= \ln\left(\frac{4(|\compfunccs(S_n)|+1)(|\compfunccs(S_n)|+2)(|\compfuncsi(S_n)|+1)(|\compfuncsi(S_n)|+2)}{\delta}\right) + |\compfunccs(S_n)|\ln 4 \\
    &\le \ln(\frac{4}{\delta}) + |\compfunccs(S_n)|\ln 4 + \frac{|\compfunccs(S_n)| + 2}{2} + \frac{|\compfuncsi(S_n)| + 2}{2} \\
    &= |\compfunccs(S_n)|\ln(4\sqrt{e}) + \frac{1}{2}|\compfuncsi(S_n)| + \ln(\frac{4e^2}{\delta})
\end{align*}
Next, 
let us abbreviate the right-hand side of the bound in
\thmref{agnosemitstableboundedsizewithinterpol}:
\begin{gather*}
    \erisk(\reconfunc(\compfunc(S_n)); S_n) \left( 5 \sqrt{\frac{8T_{\compfunc, S_n}}{n}} + 4\sqrt{\frac{8|\compfuncsi(S_n)|\ln 2}{n}} \right) + 2L\sqrt{\frac{8T_{\compfunc, S_n}}{n}} + \frac{(28+8L)T_{\compfunc, S_n}}{3n} \\
    + L\sqrt{\frac{8|\compfuncsi(S_n)|\ln 2}{n}} + \frac{28|\compfuncsi(S_n)|\ln 2}{3n} \\
    \quad := A\subI\erisk(\reconfunc(\compfunc(S_n)); S_n) + A\subII,
\end{gather*}
\begin{align*}
     A\subI &:= 5 \sqrt{\frac{8T_{\compfunc, S_n}}{n}} + 4\sqrt{\frac{8|\compfuncsi(S_n)|\ln 2}{n}} \\
     &\le 5 \sqrt{\frac{8(|\compfunccs(S_n)|\ln(4\sqrt{e}) + \frac{1}{2}|\compfuncsi(S_n)| + \ln(\frac{4e^2}{\delta}))}{n}} + 4\sqrt{\frac{8|\compfuncsi(S_n)|\ln 2}{n}} \\
     &\le 5 \sqrt{\frac{8|\compfunccs(S_n)|\ln(4\sqrt{e})}{n}} + 5 \sqrt{\frac{4|\compfuncsi(S_n)|}{n}} + 5 \sqrt{\frac{8\ln(\frac{4e^2}{\delta})}{n}} + 4\sqrt{\frac{8|\compfuncsi(S_n)|\ln 2}{n}} \\
     &= 5\sqrt{8\ln(4\sqrt{e})} \cdot \sqrt{\frac{|\compfunccs(S_n)|}{n}} + (10 + 4\sqrt{8\ln 2})\sqrt{\frac{|\compfuncsi(S_n)|}{n}} + 5\sqrt{8} \sqrt{\frac{\ln(\frac{4e^2}{\delta})}{n}} \\
     &\le 20\sqrt{\frac{|\compfunccs(S_n)|}{n}} + 20\sqrt{\frac{|\compfuncsi(S_n)|}{n}} + 15\sqrt{\frac{\ln(\frac{4e^2}{\delta})}{n}} \\
     &:= B\subI,
\end{align*}
\begin{align*}
A\subII 
&:= 
2L\sqrt{\frac{8T_{\compfunc, S_n}}{n}} 
+ \frac{(28+8L)T_{\compfunc, S_n}}{3n} 
+ L\sqrt{\frac{8|\compfuncsi(S_n)|\ln 2}{n}} 
+ \frac{28|\compfuncsi(S_n)|\ln 2}{3n} 
\\&\le 
2L\sqrt{\frac{8\left( |\compfunccs(S_n)|\ln(4\sqrt{e}) + \frac{1}{2}|\compfuncsi(S_n)| + \ln(\frac{4e^2}{\delta}) \right)}{n}} 
\\&\quad
+ \frac{(28+8L)\left( |\compfunccs(S_n)|\ln(4\sqrt{e}) + \frac{1}{2}|\compfuncsi(S_n)|
+ \ln(\frac{4e^2}{\delta}) \right)}{3n} 
\\&\quad
+ L\sqrt{\frac{8|\compfuncsi(S_n)|\ln 2}{n}} + \frac{28|\compfuncsi(S_n)|\ln 2}{3n} 
\\&\quad\le 
2L \sqrt{\frac{8|\compfunccs(S_n)|\ln(4\sqrt{e})}{n}} + 2L \sqrt{\frac{4|\compfuncsi(S_n)|}{n}} + 2L \sqrt{\frac{8\ln(\frac{4e^2}{\delta})}{n}} \\
&\quad + \frac{(28+8L)\left( |\compfunccs(S_n)|\ln(4\sqrt{e}) + \frac{1}{2}|\compfuncsi(S_n)| + \ln(\frac{4e^2}{\delta}) \right)}{3n} \\
&\quad + L\sqrt{\frac{8|\compfuncsi(S_n)|\ln 2}{n}} + \frac{28|\compfuncsi(S_n)|\ln 2}{3n} \\
&= \frac{(28+8L)\ln(4\sqrt{e})}{3}\frac{|\compfunccs(S_n)|}{n}+2L\sqrt{8\ln(4\sqrt{e})}\sqrt{\frac{|\compfunccs(S_n)|}{n}} + \\
&\quad + \frac{14+4L+28\ln 2}{3}\frac{|\compfuncsi(S_n)|}{n} + (4+\sqrt{8\ln 2})L\sqrt{\frac{|\compfuncsi(S_n)|}{n}} \\
&\quad + \frac{28+8L}{3}\frac{\ln(\frac{4e^2}{\delta})}{n} + 2\sqrt{8}L\sqrt{\frac{\ln(\frac{4e^2}{\delta})}{n}} \\
&\le (6L+18)\frac{|\compfunccs(S_n)|}{n}+8L\sqrt{\frac{|\compfunccs(S_n)|}{n}} + \\
&\quad + (2L+12)\frac{|\compfuncsi(S_n)|}{n} + 7L\sqrt{\frac{|\compfuncsi(S_n)|}{n}} \\
&\quad + (3L+10)\frac{\ln(\frac{4e^2}{\delta})}{n} + 6L\sqrt{\frac{\ln(\frac{4e^2}{\delta})}{n}} \\
&=\left(6\frac{|\compfunccs(S_n)|}{n} + 8\sqrt{\frac{|\compfunccs(S_n)|}{n}} + 2\frac{|\compfuncsi(S_n)|}{n} + 7\sqrt{\frac{|\compfuncsi(S_n)|}{n}} + 3\frac{\ln(\frac{4e^2}{\delta})}{n}+6\sqrt{\frac{\ln(\frac{4e^2}{\delta})}{n}}\right)L \\
&\quad + 18\frac{|\compfunccs(S_n)|}{n} + 12\frac{|\compfuncsi(S_n)|}{n} + 10\frac{\ln(\frac{4e^2}{\delta})}{n} \\
&:= B\subII.
\end{align*}
Let $\bmu$ be any distribution over $\X \times \Y$, and let $n \in \mathbb{N}$, and $\delta \in (0,1)$. From \thmref{agnosemitstableboundedsizewithinterpol} we know that for $S_n \sim \bmu^n$, with probability at least $1 - \delta$, if $|\compfunccs(S_n)| < \frac{n}{4}-1$ then
\begin{align}
\label{eq:boundfromtheorem}
    \risk(\reconfunc(\compfunc(S_n))) - \erisk(\reconfunc(\compfunc(S_n)); S_n) &\le
    A\subI\erisk(\reconfunc(\compfunc(S_n)); S_n) + A\subII \\
    \nonumber
    &\le B\subI\erisk(\reconfunc(\compfunc(S_n)); S_n) + B\subII.
\end{align}
Now, for $B\subI$ we note that even if $|\compfunccs(S_n)| \ge \frac{n}{4}-1$ we have:
\begin{align*}
    B\subI &:= 20\sqrt{\frac{|\compfunccs(S_n)|}{n}} + 20\sqrt{\frac{|\compfuncsi(S_n)|}{n}} + 15\sqrt{\frac{\ln(\frac{4e^2}{\delta})}{n}} \\
    &\ge 0 \ge -1 \\
\end{align*}
and for $B\subII$, we observe firstly that if $n \le 4$:
\begin{align*}
    B\subII &:= \left(6\frac{|\compfunccs(S_n)|}{n} + 8\sqrt{\frac{|\compfunccs(S_n)|}{n}} + 2\frac{|\compfuncsi(S_n)|}{n} + 7\sqrt{\frac{|\compfuncsi(S_n)|}{n}} + 3\frac{\ln(\frac{4e^2}{\delta})}{n}+6\sqrt{\frac{\ln(\frac{4e^2}{\delta})}{n}}\right)L \\
    &\quad + 18\frac{|\compfunccs(S_n)|}{n} + 12\frac{|\compfuncsi(S_n)|}{n} + 10\frac{\ln(\frac{4e^2}{\delta})}{n} \\
    &\ge \left(\frac{3}{2} + 8\sqrt{\frac{1}{4}} + 2\frac{|\compfuncsi(S_n)|}{n} + 7\sqrt{\frac{|\compfuncsi(S_n)|}{n}} + 3\frac{\ln(\frac{4e^2}{\delta})}{n}+6\sqrt{\frac{\ln(\frac{4e^2}{\delta})}{n}}\right)L \\
    &\quad + \frac{9}{2} + 12\frac{|\compfuncsi(S_n)|}{n} + 10\frac{\ln(\frac{4e^2}{\delta})}{n} \\
    &\ge L
\end{align*}
and now even if $n \ge 5$ and $|\compfunccs(S_n)| \ge \frac{n}{4}-1$:
\begin{align*}
    B\subII &:= \left(6\frac{|\compfunccs(S_n)|}{n} + 8\sqrt{\frac{|\compfunccs(S_n)|}{n}} + 2\frac{|\compfuncsi(S_n)|}{n} + 7\sqrt{\frac{|\compfuncsi(S_n)|}{n}} + 3\frac{\ln(\frac{4e^2}{\delta})}{n}+6\sqrt{\frac{\ln(\frac{4e^2}{\delta})}{n}}\right)L \\
    &\quad + 18\frac{|\compfunccs(S_n)|}{n} + 12\frac{|\compfuncsi(S_n)|}{n} + 10\frac{\ln(\frac{4e^2}{\delta})}{n} \\
    &\ge \left(6\left(\frac{1}{4} - \frac{1}{n}\right) + 8\sqrt{\frac{1}{4} - \frac{1}{n}} + 2\frac{|\compfuncsi(S_n)|}{n} + 7\sqrt{\frac{|\compfuncsi(S_n)|}{n}} + 3\frac{\ln(\frac{4e^2}{\delta})}{n}+6\sqrt{\frac{\ln(\frac{4e^2}{\delta})}{n}}\right)L \\
    &\quad + 18\left(\frac{1}{4} - \frac{1}{n} \right) + 12\frac{|\compfuncsi(S_n)|}{n} + 10\frac{\ln(\frac{4e^2}{\delta})}{n} \\
    &\ge \left(\frac{6}{20} + 8\sqrt{\frac{1}{20}} + 2\frac{|\compfuncsi(S_n)|}{n} + 7\sqrt{\frac{|\compfuncsi(S_n)|}{n}} + 3\frac{\ln(\frac{4e^2}{\delta})}{n}+6\sqrt{\frac{\ln(\frac{4e^2}{\delta})}{n}}\right)L \\
    &\quad + \frac{18}{20} + 12\frac{|\compfuncsi(S_n)|}{n} + 10\frac{\ln(\frac{4e^2}{\delta})}{n} \\
    &\ge L.
\end{align*}
Thus, even if $|\compfunccs(S_n)| \ge \frac{n}{4}-1$, we have
\begin{align}
\label{eq:othercase}
    B\subI\erisk(\reconfunc(\compfunc(S_n)); S_n) + B\subII &\ge L - \erisk(\reconfunc(\compfunc(S_n)); S_n)\\
    \nonumber
    &\ge \risk(\reconfunc(\compfunc(S_n))) - \erisk(\reconfunc(\compfunc(S_n)); S_n).
\end{align}
Finally, the theorem follows from \eqref{boundfromtheorem} and \eqref{othercase}.
\enpf

\section{Diameter-truncating BIE $\Y$
}
Recall the bounded-in-expectation
(BIE)
condition: 
$\E_{(X,Y)\sim\bmu}\ell(y_0,Y)<\infty$
for some $y_0\in\Y$.
Under BIE, the trivial $f_0(x)\equiv y_0$
achieves $R(f_0)<\infty$,
and a fortiori, $R^*=\inf_f R(f)<\infty$.
\begin{lemma}
\label{lem:y0}
If there exists {\em some}
$y_0\in\Y$ for which
$\E_{(X,Y)\sim\bmu}\ell(y_0,Y)<\infty$,
then this holds for {\em all} $y\in\Y$.
\end{lemma}
\begin{proof}
Suppose that $\E_{(X,Y)\sim\bmu}\ell(y_0,Y)<\infty$
for some 
$y_0\in\Y$. Then, 
by the triangle inequality,
for any other $y'\in\Y$, we have
\beq
\E_{(X,Y)\sim\bmu}\ell(y',Y)
\le
\E_{(X,Y)\sim\bmu}[
\ell(y',y_0)
+
\ell(y_0,Y)
]
=
\ell(y',y_0)
+
\E_{(X,Y)\sim\bmu}
\ell(y_0,Y)
<
\infty.
\eeq
\end{proof}

Thus, the choice of $y_0\in\Y$
is immaterial; let us fix one such
element once and for all.
For any sequence
$L_n\uparrow\infty$,
let $\trunc{\Y}_{n}:=B(y_0,L_n)$
denote the ``$L_n$-truncated''
space.

We observe that
\beq
f^*(x)
:=\argmin_{y'\in\Y}
\E[\ell(y',Y)\gn X=x]
\eeq
(where ties in $\Y$ are broken lexicographically) achieves
$R(f^*)=R^*$,
since it is a pointwise minimizer of the non-negative risk integrand.
Let us
also
define a truncated version:
\beq
f_n^*(x)
:=
\argmin_{\hat{y}\in\trunc{\Y}_{
n}}
\E[\ell(\hat{y},Y)\gn X=x]
.
\eeq
Since $y_0\in\trunc{\Y}_{n}$,
we have that
\beq
g_n(x):=
\E[\ell(f_n^*(x),Y|X=x)]\le 
\E[\ell(y_0,Y|X=x)]
=:h(x).
\eeq
While one or both of
$g_n,h$
may be infinite
for some $x$,
BIE implies that $h$
is integrable.
Next, we claim that
\beq
\lim_{n\to\infty}g_n(x)
=
g(x)
:=
\E[\ell(f^*(x),Y|X=x)],
\qquad
x\in\X
.
\eeq
Indeed, this follows from a stronger
property:
there is a function $N:\X\to\N$
such that
$f_n^*(x)=f^*(x)$
for all $x$
and all $n\ge N(x)$;
this is immediate by construction
and because $L_n\uparrow\infty$.
Applying Lebesgue's Dominated Convergence Theorem 
to the sequence
$g_n\le h$
yields
\begin{theorem}
\label{thm:trunc-risk}
If $\Y$ is BIE, $f^*$
is a Bayes-optimal predictor and $f^*_n$
is its truncation as defined above, then
\beq
\lim_{n\to\infty}R(f_n^*)
=R(f^*)=R^*.
\eeq
\end{theorem}

\section{Discretizing separable $(\Y,\ell)$}
\label{sec:discr}

For any separable $(\Y,\ell)$
and $\eps>0$,
any $\eps$-net $
\Y_\eps
\subseteq\Y$
is always countable.
Define 
$\pi_\eps(y)$
as the closest element to $y$ in
$
\Y_\eps
$,
breaking ties lexicographically.
Then the Voronoi cell $V(y)$
about each $y\in\Y_\eps$
is given by
$V(y)=\set{y'\in\Y:\pi_\eps(y')=y}$.
Any 
probability measure
$\bmu$ 
on the product Borel $\sigma$-algebra
on $\X\times\Y$
induces the product measure $\bmu_\eps$
on $\X\times \Y_\eps$
as follows.
By \citet[Appendix F, Theorem 1]{pollard02}, 
the measure $\bmu$
admits the disintegration
intro $\mu^\Y\otimes\Lambda$,
where $\mu^\Y$ is the $\Y$-marginal of
$\bmu$
and $\Lambda(y,E)=
\P_{(X,Y)\sim\bmu}
(X\in E\gn Y=y)
$
is the conditional kernel.
Define the
measure
$\mu^\Y_\eps$
on
$\Y_\eps$
by
$\mu^\Y_\eps(y)=
\mu^\Y(V(y))
$
and the product measure
$\bmu_\eps=\mu^\Y_\eps\otimes\Lambda$
on $\X\times\Y_\eps$.
Let $R^*$,
$R^*_\eps$
be the Bayes-optimal risk
under $\bmu$
and $\bmu_\eps$,
respectively.

\begin{theorem}
\label{thm:discr}
\beq
\lim_{\eps\to0}
R^*_\eps
=
R^*.
\eeq
\end{theorem}
\bepf
Appealing to a standard
truncation argument,
we assume without loss of generality
that $R^*<\infty$.
The product metric
$\rho\oplus\ell$
on $\Z=\X\times\Y$,
given by
$\rho\oplus\ell((x,y),(x',y'))
=
\rho(x,x')
+
\ell(y,y')
$
renders
$(\Z,\rho\oplus\ell,\bmu)$
and
$(\Z_\eps,\rho\oplus\ell,\bmu_\eps)$
metric probability spaces.
Let $h^*:\X\to\Y$ be
the Bayes-optimal predictor
for 
$(\Z,\rho\oplus\ell,\bmu)$,
and define
$f^*:\X\times\Y\to\R$
by
$f^*(x,y)=\ell(h^*(x),y)$.
Then
\beq
R^*
=
\int_{\X\times\Y}
f^*(x,y)\mathd\bmu(x,y).
\eeq
Since we assumed $R^*<\infty$,
we have that $f^*\in L_1(\bmu)$
and hence, 
by \citet[Lemma A.1]{hksw21}
$f^*$ may be approximated
in $L_1$ by Lipschitz functions:
for all $\eta>0$, there is a $\Delta<\infty$
and a $\Delta$-Lipschitz
$\tilde f:\X\times\Y\to\R$
such that
$
\int_{\X\times\Y}
|
f^*(x,y)
-
\tilde f(x,y)|\mathd\bmu(x,y)
<\eta
$.
Thus, there is no loss of generality
in assuming
$f^*$
to be $\Delta$-Lipschitz:
\beq
|
f^*(x,y)
-
f^*(x',y')
|
\le
\Delta(
\rho(x,x')
+
\ell(y,y')
).
\eeq
Define the natural projection
of $h\in\Y^\X$
onto $h_\eps\in\Y_\eps$,
via $h_\eps(x):=\pi_\eps(h(x))$.
Then
by the Lipschitz
property,
$|R(h_\eps)-R(h)|\le\Delta\eps$.
\enpf

\end{document}